\documentclass[sn-mathphys,Numbered]{sn-jnl}

\usepackage{graphicx}%
\usepackage{multirow}%
\usepackage{amsmath,amssymb,amsfonts}%
\usepackage{amsthm}%
\usepackage{mathrsfs}%
\usepackage[title]{appendix}%
\usepackage{xcolor}%
\usepackage{textcomp}%
\usepackage{manyfoot}%
\usepackage{booktabs}%
\usepackage{algorithm}%
\usepackage{algorithmicx}%
\usepackage{algpseudocode}%
\usepackage{listings}%
\usepackage{algorithm}
\usepackage{algcompatible}
\usepackage{color}
\usepackage{bbm}
\usepackage{slashbox}
\usepackage{amsthm}
\usepackage{hyperref}
\usepackage{multicol}
\usepackage{array}

\newcommand{\shadow}[1]{}
\newcommand{\blue}[1]{\textcolor{blue}{#1}}
\newcommand{\red}[1]{\textcolor{red}{#1}}

\newcommand{\brown}[1]{\textcolor{brown}{#1}}
\newcommand{\orange}[1]{\textcolor{orange}{#1}}
\DeclareMathOperator*{\argmax}{arg\,max}
\def\b{\blue}
\def\s{\shadow}
\def\r{\red}
\def\o{\orange}

\newtheorem{theorem}{Theorem}
\newtheorem{corollary}{Corollary}[theorem]

\theoremstyle{definition}
\newtheorem{definition}{Definition}[section]
\usepackage{adjustbox}
\begin{document}

\title[Article Title]{Deep Clustering with Self-Supervision using Pairwise  Similarities}


\author*[1,2]{\sur{Mohammadreza Sadeghi}}\email{mohammadreza.sadeghi@mcgill.ca}

\author[1,2]{ \sur{Narges Armanfard}}\email{narges.armanfard@mcgill.ca}

\affil*[1]{\orgdiv{Department of Electrical and Computer Engineering}, \orgname{McGill University}, \orgaddress{\street{3380 Blvd Robert-Bourassa}, \city{Montreal}, \postcode{H2X 2G6}, \state{QC}, \country{Canada}}}

\affil[2]{\orgdiv{Mila - Quebec AI Institute},  \orgaddress{\street{6666 Rue Saint-Urbain}, \city{Montreal}, \postcode{H2S 3H1}, \state{QC}, \country{Canada}}}


\abstract{Deep clustering incorporates embedding into clustering to find a lower-dimensional space appropriate for clustering. In this paper, we propose a novel deep clustering framework with self-supervision using pairwise similarities (DCSS). The proposed method consists of two successive phases. In the first phase, we propose to form hypersphere-like groups of similar data points, i.e. one hypersphere per cluster, employing an autoencoder that is trained using cluster-specific losses. The hyper-spheres are formed in the autoencoder's latent space. In the second phase, we propose to employ pairwise similarities to create a $K$-dimensional space that is capable of accommodating more complex cluster distributions, hence providing more accurate clustering performance. $K$ is the number of clusters. The autoencoder's latent space obtained in the first phase is used as the input of the second phase. The effectiveness of both phases is demonstrated on seven benchmark datasets by conducting a rigorous set of experiments. The DCSS code is available: \href{https://github.com/Armanfard-Lab/DCSS}{{\underline{https://github.com/Armanfard-Lab/DCSS}}}}




\maketitle

\section{Introduction}\label{sec:introduction}

Many science and practical applications, information about category (aka label) of data samples is non-accessible or expensive to collect. Clustering, as a major data analysis tool in pattern recognition and machine learning, endeavors to gather essential information from unlabeled data samples.  The main goal of clustering methods is to partition data points based on a similarity metric.

Deep learning-based clustering methods have been widely studied, and their effectiveness is demonstrated in many applications such as image segmentation \cite{ijcnn14}, social network analysis \cite{ijcnn15}, face recognition \cite{ijcnn16,zhu2022local}, and machine vision \cite{ijcnn17,kong2022human}. The common practice in these methods is to map the original feature space onto a lower dimensional space (aka latent space) in which similar samples build data groups that can be detected by a simple method like k-means \cite{kmeans}. 

One of the most common approaches in obtaining the lower dimensional space is based on autoencoder (AE) and its variations \cite{ijcnn18,ijcnn19,ijcnn20,haseeb2022autoencoder,icip}. An AE consists of two networks: an encoder and a decoder. The encoder maps the original input space onto a latent space while the decoder tries to reconstruct the original space using the encoder's output space. Encoder and decoder networks are trained to minimize a loss function that contains the data reconstruction losses. The AE's latent space, whose dimension is much lower than the dimension of the original input space, is indeed a nonlinear transformation of the original space.

Some more advanced AE-based clustering methods, e.g. \cite{dkm,dcn,idec,dsc,idecf}, include in their loss function the data clustering losses besides the reconstruction losses. This makes the AE's latent space more effective for data clustering.
Despite the reconstruction loss that can be directly computed based on the difference between the encoder's input and the decoder's output, calculating the \emph{true} clustering loss is impossible due to the unsupervised nature of the clustering problem where the true cluster label of the data points remains unknown during the training phase.
Hence, researchers employ an \emph{approximation} of the clustering losses when training the networks. At each training iteration, the approximation is calculated based on the data distribution in the latent space obtained in the previous iteration.

\begin{figure}[t]
  \centering
  \includegraphics[width=0.85\linewidth]{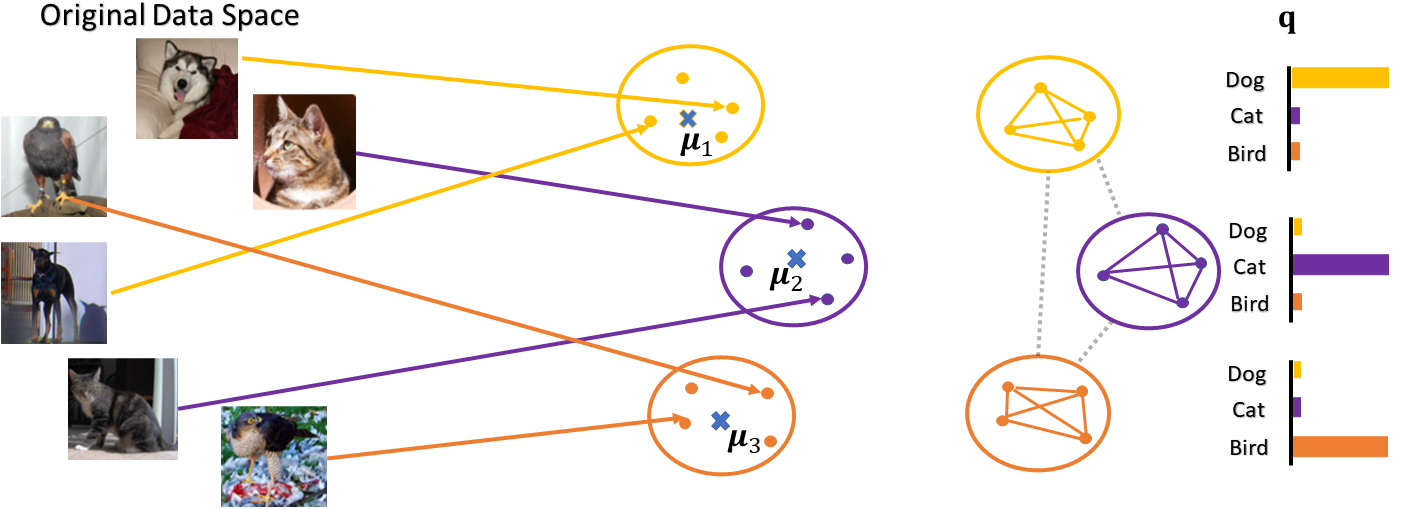}
\caption{The motivation of the proposed DCSS method. Arrows show the nonlinear mapping, using the AE, from the original input space to the AE's latent space (i.e. the $\mathbf{u}$ space).\s{ In the first phase of DCSS, we aim to gather the data points near their corresponding group centers in the $\mathbf{u}$ space. In the second phase,} DCSS employs pairs of similar and dissimilar samples to create the $K$-dimensional space $\mathbf{q}$ in which pairwise similarities and dissimilarities are strengthened. Similar samples are connected with solid lines, and dashed lines represent dissimilar data.\s{ Data points within the same cluster are connected with the same color, and dashed lines represent dissimilarities between different clusters.}}
\label{fig11}
\end{figure}

Many existing algorithms in the field, e.g. \cite{dkm,dcn,opochinsky2020k,caron2018deep,tian2017deepcluster,pan2021image}, approximate the data clustering losses by first performing a crisp cluster assignment and then calculating the clustering losses using the crisply clustered data, based on a criterion such as the level of compactness or density of the data within clusters. Crisp assignment (as opposed to soft assignment) assigns a data point to only one cluster -- e.g., to the one with the closest center. However, such crisp cluster assignment of the data in an intermediate training iteration may mislead the training procedure if a significant number of samples are mis-clustered, as the error propagates to the following iterations. This crisp assignment issue would be more serious if there exists a high uncertainty when deciding to which cluster a data point should be assigned -- e.g., the extreme uncertainty would be related to the case where a data point is equally close to all cluster centers. \b{Recently, contrastive learning \cite{chen2020simple,c3,cc} has garnered significant attention within the domain of unsupervised representation learning and clustering. For instance, the Simple Framework for Contrastive Learning for Visual Representation Learning (SimCLR) \cite{chen2020simple} and Contrastive Clustering (CC) \cite{cc} have departed from the conventional practice of relying on pseudo labels for training neural networks. Instead, these methods generate two augmentations of a given data sample and treat the two augmentations of the same sample as positive pairs while designating all other samples in the batch as negative pairs, even if they correspond to the same object in different images. This approach, which treats the same objects in two different images as negative pairs, may inadvertently impact the network's ability to converge towards an optimal solution for the problem.}

To the best of our knowledge, only a limited number of deep clustering algorithms, such as \cite{dacc, dccm, ddc, scan}, attempt to utilize the pairwise relationships between sample pairs in an unsupervised manner. However, it should be noted that some of these methods, including \cite{dacc, ddc}, encounter challenges related to the presence of a high number of false positive pairs during their model training, primarily because they design their loss functions to include more and more samples till it includes all samples rather than focusing on highly confident samples.
All other deep clustering algorithms neglect the important relevant information available in sample pairs while the effectiveness of such information has been proven in the \emph{supervised} and \emph{semi-supervised} learning methods \cite{kaya2019deep,wang2017deep,xing2002distance,schroff2015facenet,xie2020pairwise} where the data class labels are employed during the training phase.
For example, metric learning algorithms are supervised learning techniques that learn a distance metric employing pairwise distances \cite{kaya2019deep,wang2017deep}. Their goal is to decrease the distance between similar samples, i.e. samples from the same class, and increase the distance between dissimilar samples, i.e. samples from different classes.

Furthermore, to the best of our knowledge, the existing clustering methods utilize a single loss function for all data clusters, ignoring the possible existence of differences between the characteristics of the different clusters. \b{For instance, consider a dataset comprising images of vehicles where one cluster represents compact cars and another represents trucks. The intra-cluster variability—such as the range of shapes, sizes, and designs—can be significantly different between these two categories.} We present novel cluster-specific losses in this paper, designed to not only emphasize the reconstruction and centering of data samples around cluster centers but also to aid the AE in finding distinctive properties within each cluster.  The only existing algorithm that explicitly employs cluster-specific losses is presented in \cite{opochinsky2020k}. However, this algorithm suffers from high computational cost as the algorithm requires training of K distinct AEs, where K is the number of data clusters.

In summary, to the best of our knowledge, all existing AE-based unsupervised clustering methods suffer from at least one of the following: the crisp assignment issue, ignoring the relevant, useful information available in the data pairs, failure to reliably identify similar and dissimilar samples, treating all clusters similarly through minimizing a single common loss function for all data clusters.

In this paper, we propose a novel AE-based clustering algorithm called Deep Clustering with Self-Supervision, DCSS, that addresses all the aforementioned drawbacks. 
DCSS is a novel unified framework that employs pairwise similarities as a means of self-supervision during its training procedure. DCSS mitigates the error propagation issue caused by the uncertain crisp assignments by employing the soft assignments in the loss function. DCSS considers an individual loss for every data cluster, where a loss consists of weighted reconstruction and clustering errors. A sample's clustering error is calculated using the sample's Euclidean distance to the cluster centers. This results in obtaining a latent space (for the AE), called $\mathbf{u}$ space, in which similar data points form hypersphere-like clusters, one hypersphere per cluster.\s{Hence, the AE's latent space, called $\mathbf{u}$ space, is trained to provide pure spheres of data clusters -- i.e., one sphere per cluster in the $\mathbf{u}$ space.}
To make the cluster distributions more distinguishable from each other and to accommodate more complex distributions, we propose to employ pairwise similarities and train a $K$-dimensional space $\mathbf{q}$ in which a pair of similar (dissimilar) samples sit very close to (far from) each other, where K is the number of data clusters and inner product is used for similarity measurement.
We define $\mathbf{q}$ to be the last layer of a fully connected network called MNet. Only similar and dissimilar pairs of samples contribute to training MNet.  %
Due to the curse of dimensionality \cite{friedman1997bias}, similar and dissimilar samples are not recognizable in the original input feature space. Instead, we propose to measure the pairwise similarities in the partially trained $\mathbf{u}$ and $\mathbf{q}$ spaces, which are more reliable\footnote{In this paper, space A is considered more reliable than space B if and only if the clustering performance in A is better than in B.} for similarity measurement.
The input layer of MNet is the output layer of the AE's encoder network, i.e., the $\mathbf{u}$ space. Our experiments, supported by mathematical proofs, demonstrate that the data representations in the $\mathbf{q}$ space are very close to one-hot vectors where the index of the most active element points the true cluster label out.
Furthermore, we demonstrate that the DCSS method can be employed as a general framework to improve the performance of the existing AE-based clustering methods, e.g. \cite{dec,idec,dcn,dkm}.
An intuitive motivation of the proposed method is illustrated in Fig. \ref{fig11}.

The rest of this paper is organized as follows. Section \ref{RL work} presents a brief review of deep clustering methods. Section \ref{proposed method} presents details of the proposed DCSS framework. Extensive experimental results that demonstrate the effectiveness of the DCSS method are presented in Section \ref{experiments}. Section \ref{conclusion} conveys the gist of this paper.

\section{Related Work} \label{RL work}
\s{\subsection{conventional clustering approaches:}
Clustering has been widely studied in machine learning from different aspects such as feature selection \cite{boutsidis2009unsupervised, liu2005toward, alelyani2018feature}, distance metric \cite{xing2002distance, xiang2008learning}, and categorizing methods \cite{macqueen1967some, von2007tutorial, li2004entropy}.
 K-means \cite{kmeans} and fuzzy c-means \cite{fcm} are the two popular conventional methods that are applicable to a wide range of tasks \cite{ijcnn10,ijcnn11,ijcnn12}. However, because of their distance metric, they can only extract local information of the data dealing with high dimensional feature spaces. Some conventional algorithms, such as \cite{ye2007discriminative}, aim to handle this difficulty by jointly performing subspace selection and a clustering algorithm in an iterative manner. At each iteration, they group data points using k-means and endeavor to maximize the inter-cluster variance employing the data projected in a lower-dimensional space. They repeated this process until convergence. Another group of conventional methods, called spectral clustering, such as \cite{ijcnn28,ijcnn29}, address the high input dimension issue by embedding the high dimensional data into a lower-dimensional space. They then apply clustering algorithms in the new space. For their embedding phase, first, they construct a weighted graph in which nodes are data samples and weights are defined based on pairwise relationships between them in the original space. Then, they define a minimization problem using the Laplacian matrix of the weighted graph. Although they could surpass the clustering performance of k-means in different applications, the complexity of solving the optimization problem limited the application of these algorithms to only small datasets. In order to make spectral clustering more applicable on large datasets, \cite{ijcnn30,ijcnn31} propose stochastic optimization methods that try to estimate the original optimization problem. \brown{However, since the Laplacian is defined on the high-dimensional original space and in this space, the distance between two points cannot properly represent the similarity or dissimilarity of data points, these conventional methods are not successful when dealing with complex datasets.}
 
 \subsection{Deep-learning-based approaches}}
An exhaustive review of previous works is beyond
the scope of this paper. We refer to the survey of Xu et al. \cite{xu2015comprehensive} on non-deeplearning-based clustering methods. The following focuses on the review of some related deep clustering methods. 

\subsection{Deep learning based methods}

Deep neural networks (DNN) have been widely used to tackle the unsupervised clustering problem. These algorithms try to train a DNN-based model in an unsupervised manner \cite{dml,Niu2021SPICESP, scan, shen2021you, huang2022learning}. 
\b{\cite{scan} utilizes self-supervised learning to create meaningful image feature representations without relying on labeled data. These learned features serve as a foundation for a unique clustering method. The technique involves identifying closest neighbors based on feature similarity and grouping each image with its neighbors, aiming to ensure semantic similarity within clusters and distinctiveness across clusters. However, a limitation of this approach is the computational and memory costs associated with mining nearest neighbors for each image, particularly with extensive datasets. \cite{Niu2021SPICESP} breaks down the clustering task into three phases: first, training the feature model to evaluate instance similarities; second, training the clustering head by minimizing the Cross-Entropy loss, which measures the difference between the predicted probability distribution over clusters for each instance and the distribution represented by the pseudo-labels; and third, conducting joint optimization using reliable pseudo-labeling to enhance both clustering precision and feature representation quality. \cite{huang2022learning} utilizes prototype scattering loss (PSL) to increase the separation between prototypical representations of distinct clusters. It also employs positive sampling alignment (PSA) to align an augmented view of an instance with the sampled neighbors of another view, which are considered genuine positive pairs, aiming to improve compactness within clusters.} \cite{hu2017learning} encourages predicted representations of the augmented data points to be close to those of the original data points by maximizing the information-theoretic dependency between data and their predicted representations. RUC \cite{park2021improving} proposes a two-step method where, in the first step, it endeavors to clean the dataset, and in the second step, it retrains the network with the purified dataset. \cite{dsec} trains a DNN in an unsupervised manner to extract an indicator feature vector for each data sample. It then uses the obtained vectors to assign the data points to different clusters.
Recently, contrastive learning has attracted researchers' attention in the unsupervised clustering field \cite{cc,chen2020simple,li2022twin}. As is discussed in Section \ref{sec:introduction}, such algorithms first need to construct negative and positive pairs by applying augmentation to the data points. They then map the data into a feature space and endeavor to maximize similarity (minimize dissimilarity) in positive (negative) pairs. An extensive review of the DNN-based methods can be found in \cite{min2018survey}.

Among the DNN-based models, the AE-based and generative-based algorithms have been widely studied and used for unsupervised data clustering. These two categories are reviewed in the following sections.  \s{For example, contrastive clustering (CC) \cite{cc} defines instance- and cluster-level losses, respectively, on rows and columns of a feature space in order to maximize similarity while minimizing dissimilarity.}   

\s{\o{Here we categorize the deep clustering methods to three groups: AE-based methods, generative-based methods, and other deep neural network (DNN)-based methods that doesn't fall in the first two groups.}
Deep clustering utilizes deep neural networks to find a suitable space for data clustering tasks. \brown{Deep clustering algorithms fall into three main categories: AE-based, generative model-based, and deep neural network (DNN) based algorithms}}

\subsubsection{AE-based algorithms}
AE-based algorithms utilize deep autoencoders to embed original data points in a  lower-dimensional space. In some algorithms, such as \cite{huang2014deep,ijcnn32}, learning the lower representation of the data points is separated from the clustering task. In \cite{huang2014deep}, an AE is used to find a lower-dimensional representation of data points by enforcing group sparsity and locality-preserving constraints. The cluster assignments are then obtained by applying the k-means algorithm to the obtained lower-dimensional space. Graph clustering \cite{schaeffer2007graph,nascimento2011spectral} is a key branch of clustering that tries to find disjoint partitions of graph nodes such that the connections between nodes within the same partition are much denser than those across different partitions. \cite{ijcnn32} takes advantage of a deep autoencoder to find a lower-dimensional representation of a graph; it then utilizes the k-means algorithm to define clusters in the lower-dimensional space.
 
In order to further improve clustering performance, more recent AE-based algorithms simultaneously embed data points in a lower-dimensional feature space and perform clustering using the obtained space. Deep embedded clustering (DEC) \cite{dec} first trains a stacked autoencoder layer by layer using the reconstruction losses and then removes the decoder and updates the encoder part by minimizing a Kullback–Leibler (KL) divergence between the distribution of soft assignments and a pre-determined target distribution. Soft assignments are the similarity between data points and cluster centers and are calculated using Students' t-distribution. Due to the unsupervised nature of the clustering problem, the target distribution of the data points is unknown. Hence, DEC uses an arbitrary target distribution, which is based on the squared of the soft assignments.
Despite the DEC method, a few recent studies, e.g \cite{dcn,dkm,idec}, propose to take advantage of the AE's decoder as well as the encoder. These algorithms use notions of both reconstruction and clustering losses with the goal of maintaining the local structure of the original data points while training the algorithm's networks. For example, improved deep embedding clustering (IDEC) \cite{idec} tries to improve the clustering performance of DEC by considering the reconstruction loss of an AE besides the KL divergence loss of DEC.\s{Deep convolutional embedded clustering (DCEC) \cite{guo2017deep} could enhance the performance of IDEC by changing the fully connected structure of IDEC to a deep convolutional autoencoder. Moreover, DCEC proposed an end-to-end pre-training scheme by minimizing the reconstruction loss instead of pre-training a stacked autoencoder proposed in DEC and IDEC.}\s{ Some other works, instead of applying self-training using soft cluster assignments, improve the clustering performance of DEC by proposing an independent approach \b{for finding} the target distributions. For example,} 
Improved deep embedding clustering with fuzzy supervision (IDECF) \cite{idecf}  improves the DEC method by employing both reconstruction and clustering losses and estimating the target distribution through training a deep fuzzy c-means network. Deep clustering network (DCN) \cite{dcn} jointly learns a lower-dimensional representation and performs clustering.\s{ aims to find a new representation of data points in which data points are separable by applying the k-means algorithm. To this end,} DCN trains its AE by minimizing a combination of the reconstruction loss and the objective function of the k-means algorithm. This results in a k-means-friendly latent space. DCN updates AE's parameters and cluster centers separately. The latter is based on solving a discrete optimization problem. In the deep k-means (DKM) algorithm \cite{dkm}, which has the same objective function as DCN, both network parameters and cluster centers are updated simultaneously by minimizing its objective function using stochastic gradient descent. 

Spectral clustering \cite{ijcnn28,ijcnn29} is a clustering approach that is based on building a graph of data points in the original space and then embedding the graph into a lower-dimensional space in which similar samples sit close to each other. Spectral clustering has been employed in DNN-based methods \cite{dsc,shaham2018spectralnet}. For example, deep spectral clustering (DSC) is recently presented in \cite{dsc}. DSC has a joint learning framework that creates a low-dimensional space using a dual autoencoder that has a common encoder network and two decoder networks. The first decoder tries to reconstruct the original input from the AE's latent space, and the second decoder endeavors to denoise the encoder's latent space. DSC considers reconstruction, mutual information, and spectral clustering losses for networks' training.
\subsubsection{Deep generative based algorithms}

Variational autoencoders (VAEs) \cite{kingma2013auto} and Generative adversarial networks (GANs) \cite{goodfellow2014generative} are among the most well-known deep generative models which are effective for data clustering. For example, variational deep embedding (VaDE) \cite{jiang2016variational} finds a latent space that captures the data statistical structure that can be used to produce new samples. The data generative process in VaDE is based on a Gaussian Mixture Model (GMM) and a deep neural network. Deep adversarial clustering  \cite{dac} is another generative model that applies the adversarial autoencoder \cite{44904} to clustering. The adversarial autoencoder employs an adversarial training procedure to match the aggregated posterior of the latent representation with a Gaussian Mixture distribution. \s{To this end, during the training procedure, }\cite{dac} objective function includes a reconstruction term, Gaussian mixture model likelihood, and the adversarial objective. More generative-based models can be found in \cite{min2018survey}.
GAN is a method of training a generative model by framing the problem as a supervised learning task with two sub-models: the generator model that is trained to generate new samples and the discriminator model that tries to classify examples as either real or fake (generated). Many GAN-based algorithms have been developed for clustering tasks \cite{min2018survey,cao2022unsupervised}. \cite{chen2016infogan} presents a GAN-based algorithm that learns disentangled representations in an unsupervised manner. It maximizes the mutual information between a small subset of the latent variables and the observation. \cite{yu2018mixture} expands the idea of GMM to the GAN mixture model (GANMM) by devising a GAN model for each cluster. GAN-based clustering algorithms suffer from vanishing gradients and mode collapse.

\s{\brown{Deep generative models have been studied in many research works because of their abilities in capturing the data distribution by neural networks, from which new samples could be generated. Generative adversarial networks (GANs) and variational autoencoders (VAEs) are among the most well-known deep generative models which could be applied to various tasks \cite{ehsan2017infinite,kingma2014semi,jiang2016variational,yu2018mixture}.
\b{GAN is }

For example, Variational Deep Embedding (VaDE) \cite{jiang2016variational} \b{has} an unsupervised training approach which is designed for the clustering task. VaDE finds a \b{latent space that captures} the statistical structure of the data and is capable of producing new samples. The data generative process in VaDE is modeled with a Gaussian Mixture Model (GMM) and a deep neural network. The GMM first chooses a cluster from which the latent embedding is generated and then the latent embedding is decoded using a DNN \r{to the observable? WHAT do you mean by observable??}.}

\brown{Deep adversarial clustering (DAC) \cite{dac} is another generative model for data clustering. It makes use of adversarial autoencoders that applies an adversarial training scheme to match the aggregated posterior of the latent representation with the GMM distribution. To this end, during the training procedure, DAC endeavors to minimize three losses: 1-data reconstruction loss, 2- GMM likelihood, and 3- adversarial loss. The last two losses could be considered as a clustering loss. 
GAN is another popular deep generative model that recently attract researchers' attention \cite{gui2020review}. GANs perform a min-max game between a generator and a discriminator network. The generator network aims to generate fake samples and the discriminator tries to recognize the distribution of fake samples from that of real samples. Many GAN-based clustering algorithms have been developed for clustering task \cite{min2018survey}. For instance, \cite{chen2016infogan} is an effective unsupervised method that maximizes the mutual information between noisy variables of a GAN structure and the observation. As another example, \cite{yu2018mixture} aimed to expand the idea of GMM to GAN mixture model (GANMM) by devising a GAN model for each cluster. GAN-based clustering algorithms suffer from vanishing gradients and mode collapse.
}}

\s{\subsubsection{DNN-based algorithms}
\brown{DNN-based algorithms have been designed to tackle clustering problem in many data-driven applications. These algorithms try to train a DNN model in an unsupervised manner. For instance, \cite{hu2017learning} motivates the predicted representations of an augmented data points to stay close to that of the original data point by maximizing the mutual information dependency between the data point and its representations. As another example, RUC \cite{park2021improving} proposed a new robust learning training approach that could enhance clustering performance of the existing clustering methods. RUC is a two steps method where in the first step RUC endeavors to obtain a clean dataset and in the second step, it retrain with the purified dataset.}   
Contrastive learning algorithms, such as \cite{cc,chen2020simple}, have been widely attracted researchers' attention in the recent year due to their promising performance. As is discussed in Section \ref{sec:introduction}, they first construct negative and positive pairs by applying data augmentation on data points. Then, they map them in the feature space and endeavor to maximize similarity (minimize dissimilarity) between positive (negative) pairs. For example, contrastive clustering (CC) \cite{cc} defines instance- and cluster-level losses respectively on rows and columns of the feature space in order to maximize similarity while minimizing dissimilarity.  } 

\section{Proposed Method}\label{proposed method}
\begin{figure*}[t]
  \centering
  \includegraphics[width=0.70\linewidth]{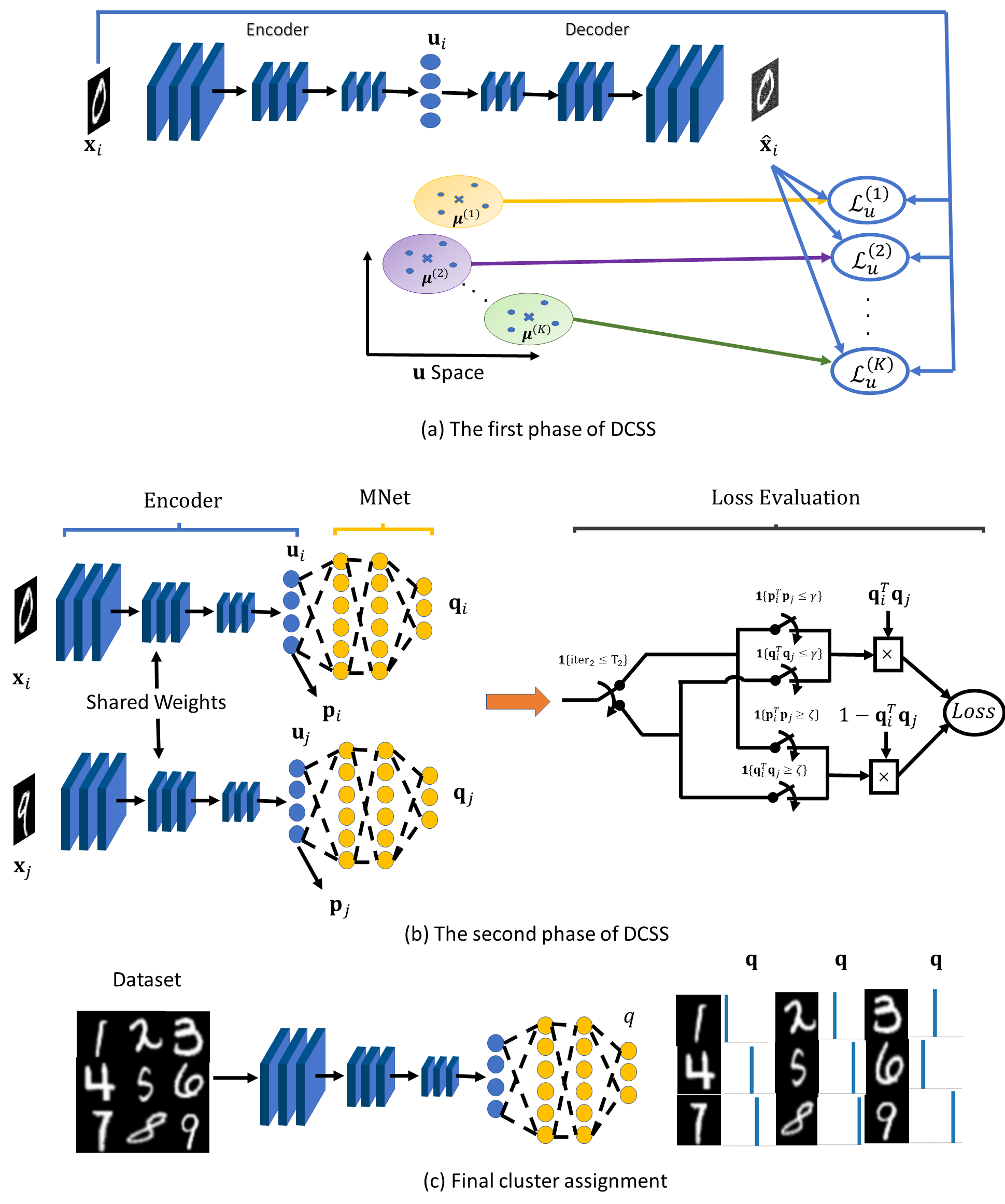}
\caption{(a) Training scheme of the first phase of DCSS.  (b) Training procedure of the second phase of DCSS; at the outset, when $\text{iter}_2\leq T_2$, MNet is trained based on the pairwise similarities defined in the $\mathbf{u}$ space -- i.e. The similarity between two data points $\mathbf{x}_i$ and $\mathbf{x}_j$ is determined using the dot product of $\mathbf{p}_i$ and $\mathbf{p}_j$. At the later stages of MNet training, when $\text{iter}_2> T_2$, the pairwise similarities are measured in the $\mathbf{q}$ space itself using $\mathbf{q}_i^T\mathbf{q}_j$. (c) Visualization of the final cluster assignment using DCSS; after completing the training phases shown in (a) and (b), we cluster a data point by locating the largest element of its representation in the $\textbf{q}$ space.}
\label{Model}
\end{figure*}

Consider a K-clustering problem that aims to partition a given dataset $X = \{\mathbf{x}_1,\mathbf{x}_2,...\mathbf{x}_N\}$ into K disjoint clusters, where $\mathbf{x}_i$ indicates the ith data sample, N is the number of data points, and K is a predefined user-settable parameter. DCSS utilizes an AE consisting of an encoder and a decoder network, respectively denoted by $f(.)$ and $g(.)$. Latent representation of $X$ is denoted by $U=\{\mathbf{u}_1,\mathbf{u}_2,...,\mathbf{u}_N\}$, where $\mathbf{u}_i = f(\mathbf{x}_i;\boldsymbol{\theta}_e) \in \mathbb{R}^d$, $d$ indicates dimension of the latent space, and $\boldsymbol{\theta}_e$ denotes parameters of the encoder network. The reconstructed output of the AE is denoted by $\hat{\mathbf{x}}_i = g(\mathbf{u}_i;\boldsymbol{\theta}_d)$, where $\boldsymbol{\theta}_d$ represents the decoder parameters. The center of the kth data group in the $\mathbf{u}$ space is denoted by $\boldsymbol{\mu}^{(k)}$.
To accommodate complex cluster distributions, we propose to employ pairwise similarities in DCSS. To this end, we employ the\s{As aforementioned, to investigate the pairwise relationship between the data points, we propose to employ a} fully connected network MNet which takes the latent representation of each data point, i.e. $\mathbf{u}_i$, as input and maps it to a $K$-dimensional vector $\mathbf{q}_i$ which its kth element indicates the probability of $\mathbf{x}_i$ belonging to the kth data cluster. In this paper, the output of MNet for the ith data point is denoted by $\mathbf{q}_i = M(\mathbf{u}_i;\boldsymbol{\theta}_M)$, where $M(.)$ and $\boldsymbol{\theta}_M$ respectively shows MNet and its corresponding parameters.  
\\
The proposed DCSS method consists of two phases. The first phase is to provide hypersphere-like data clusters through training an AE using weighted reconstruction and centering losses, and the second phase is to employ pairwise similarities to self-supervise the remaining training procedure.  

\subsection{Phase 1: AE training}
\label{step1}
\s{To obtain the reliable low dimensional space $\mathbf{u}$ in which identification of the similar and dissimilar samples is possible, we propose to train an AE with a novel and effective loss function consisting of weighted reconstruction and centering losses.}
At each training batch $\mathfrak{B}$, we propose to train the AE in K successive runs, where at each run, a specific loss corresponding to a specific data cluster is minimized. More specifically, at the kth run, the AE focuses on the reconstruction and centering of the data points that are more probable to belong to the kth data cluster.

\b{The} loss function of the kth run, i.e. $\mathcal{L}_u^{(k)}$, is shown in \eqref{eq1} where $\mathcal{L}_r^{(k)}$ and $\mathcal{L}_c^{(k)}$, shown in \eqref{eq666} and \eqref{eq4444}, respectively denotes weighted summation of the sample reconstruction and centering losses.  $\alpha$ is a hyperparameter indicating the importance of centering loss vs. reconstruction loss. $m$ indicates the level of fuzziness and is set to 1.5 in all experiments. 
\begin{subequations}\label{new_eq1}
\begin{eqnarray}
 & \mathcal{L}_u^{(k)} = \mathcal{L}^{(k)}_r +\alpha \mathcal{L}^{(k)}_c \label{eq1}\\
 & \mathcal{L}^{(k)}_r = \sum_{\mathbf{x}_i\in \mathfrak{B}} p_{ik}^m ||\mathbf{x}_i-\hat{\mathbf{x}}_i||_2^2 \label{eq666}\\
 & \mathcal{L}^{(k)}_c =\sum_{\mathbf{x}_i\in \mathfrak{B}} p_{ik}^m ||\mathbf{u}_i-\boldsymbol{\mu}^{(k)}||_2^2 \label{eq4444}
\end{eqnarray}
\end{subequations} 
\b{Where}
\begin{align}
    & p_{ik} = \frac{\frac{1}{||\mathbf{u}_i-\boldsymbol{\mu}^{(k)}||_2^{2/(m-1)}}}{\sum_{j=1}^{K} \frac{1}{||\mathbf{u}_i-\boldsymbol{\mu}^{(k)}||_2^{2/(m-1)}}}\label{eq3}
\end{align}

Since data clustering is an unsupervised task, the data cluster memberships are unknown at the problem's outset. As such, at the kth run, we use the Euclidean distance between $\mathbf{u_i}$ and  $\boldsymbol{\mu}^{(k)}$ as a means of measuring the membership degree of $\mathbf{x}_i$ to the kth data cluster, denoted by $p_{ik}$ defined in \eqref{eq3} where $\mathbf{p}_i=[p_{i1}, \ldots, p_{iK}]$.
The cluster memberships are used as the sample weights in \eqref{eq666} and \eqref{eq4444}. The closer a sample is to the cluster center $\boldsymbol{\mu}^{(k)}$, the higher contribution that sample has in minimizing the loss function corresponding to the kth run.

\s{At the kth run, to motivate the $\text{AE}$ to concentrate on the kth group data points, in $\mathcal{L}_r^{(k)}$ and $\mathcal{L}_c^{(k)}$, higher weights are assigned to the samples closer to the center $\boldsymbol{\mu}^{(k)}$. Membership of the ith data point to the kth data group, in the $\mathbf{u}$ space, is shown as $p_{ik}$ defined in (\ref{eq3}). $p_{ik}$ is used as the weight of the ith sample reconstruction and centering losses at the kth run of DCSS$_u$.} 

Every $T_1$ number of training epochs, we update the centers to the average of weighted samples in the $\mathbf{u}$ space, as is shown in (\ref{eq2}), where samples closer to $\boldsymbol{\mu}^{(k)}$ have more contribution to updating. 
\begin{align}
\boldsymbol{\mu}^{(k)} = \frac{\sum_{\mathbf{x}_i\in X}p_{ik}^m \mathbf{u}_i}{\sum_{\mathbf{x}_i\in X}p_{ik}^m} \label{eq2}
\end{align}

The block diagram of the first phase is shown in Fig. \ref{Model}(a). As is demonstrated in our experiments (see Section \ref{t-SNE visualization}), minimizing \eqref{eq1} results in forming hypersphere-like groups of similar samples in the $\mathbf{u}$ space, one hypersphere per cluster. 
\s{Move this to Experiments: A preliminary version of DCSS$_u$ is presented in \cite{ijcnn}.} \s{HERE I removed the name DCSS$_u$. It is confusing to have two names at this section. }

\subsection{Phase 2: self-supervision using pairwise similarities}\label{2&3}

To allow accommodating non-hypersphere shape distributions and to employ the important information available in the pairwise data relations, we propose to append a fully connected network, called MNet, to the encoder part of the AE, trained in Phase 1 while discarding its decoder network.\s{After completing phase 1, we discard the decoder part of the autoencoder and append the fully connected network MNet to the trained encoder.\s{ We take pairwise similarities to supervise MNet's training phase.}} 
The MNet's output layer, i.e. the $\mathbf{q}$ space, consists of K neurons where each neuron corresponds to a data cluster. We utilize the soft-max function at the output layer to obtain probability values employed for obtaining the final cluster assignments. More specifically, for an input sample $\mathbf{x}_i$, the output value at the jth neuron, i.e. $q_{ij}$, denotes the probability of $\mathbf{x}_i$ belonging to the jth cluster. 

MNet aims to strengthen (weaken) similarities of two similar (dissimilar) samples.
MNet parameters, i.e., $\boldsymbol{\theta}_M$, are initialized with random values. Hence, at the first few training epochs, when $\mathbf{q}$ is not yet a reliable space, pairwise similar and dissimilar samples are identified in the $\mathbf{u}$ space. Then, after a few training epochs, pairs of similar and dissimilar samples are identified in the $\mathbf{q}$ space.
In both of the $\mathbf{u}$ and $\mathbf{q}$ spaces, we define two samples as similar (dissimilar) if the inner product of their corresponding cluster assignment vectors is greater (lower) than threshold $\zeta$ ($\gamma$). 
More specifically, knowing that the kth element of $\mathbf{p}_i$ ($\mathbf{q}_i$) denotes the membership of $\mathbf{x}_i$ to the kth cluster in the $\mathbf{u}$ ($\mathbf{q}$) space, the inner product of $\mathbf{p}_i$ ($\mathbf{q}_i$) and $\mathbf{p}_j$ ($\mathbf{q}_j$) is considered as the notion of similarity between data points $\mathbf{x}_i$ and $\mathbf{x}_j$.

The loss function proposed for the MNet training, at the first $T_2$ training epochs, is shown in (\ref{eq4}) where
$\zeta$ and $\gamma$ are two user-settable hyperparameters, and $\mathbbm{1}\{.\}$ is the indicator function.
\begin{align}\label{eq4}
    \mathcal{L}_M =\sum_{\mathbf{x}_i,\mathbf{x}_j\in \mathfrak{B}} \mathbbm{1}\{ \mathbf{p}_i^T \mathbf{p}_j \geq \zeta\}(1-\mathbf{q}_i^T \mathbf{q}_j)+\mathbbm{1}\{ \mathbf{p}_i^T \mathbf{p}_j \leq \gamma\}(\mathbf{q}_i^T \mathbf{q}_j)
\end{align}
As can be inferred from (\ref{eq4}), only similar and dissimilar samples, identified in the $\mathbf{u}$ space, contribute to the MNet training and a pair of samples with a similarity value between $\zeta$ and $\gamma$, i.e. in the ambiguity region, does not contribute to the current training epoch. Therefore, minimizing $\mathcal{L}_M$ strengthens (weakens) the similarity of similar (dissimilar) samples in the $\mathbf{q}$ space. Along with training the MNet parameters, the encoder parameters $\boldsymbol{\theta}_e$ are also updated through back-propagation in an end-to-end manner. After completing each training epoch, centers $\boldsymbol{\mu}^{(k)}, k=1, \ldots, \text{K}$, are also updated using \eqref{eq2}.
 \begin{algorithm}[t]
\caption{ Clustering procedure using 
DCSS}\label{alg1}
\begin{algorithmic}[1]

\Statex{\textbf{Input:} Data points $X$, $\boldsymbol{\theta}_e$, $\boldsymbol{\theta}_d$, $\boldsymbol{\theta}_M$,  $\boldsymbol{\mu}^{(k)}$ for $k=1,\ldots,\text{K}$}

\Statex{\textbf{Output: } $\boldsymbol{\theta}_e$, $\boldsymbol{\theta}_M$ }
\Statex{}
\Statex{\textbf{Phase 1:}}
\STATE Initialize  $\boldsymbol{\theta}_e$ and $\boldsymbol{\theta}_d$ with a pre-trained network (see Section 2 of the supplementary material).
\FOR {$\text{iter}_1 \in \{1,2,...,\text{MaxIter}_1\}
$}
 
\FOR {$k \in \{1,2,...,\text{K}\}$}
\STATE Compute $p_{ik}$ using \eqref{eq3}, for $\, \, i\in \mathfrak{B}$
\STATE Update $\text{AE}$'s  parameters by employing \eqref{eq1} as loss function 
\ENDFOR
\STATE
Every $T_1$ iterations, update cluster centers using \eqref{eq2}
\ENDFOR
\Statex{}
\Statex{\textbf{Phase 2:}}
\FOR {$\text{iter}_2 \in \{1,2,...,\text{MaxIter}_2\}$ }
\IF {$\text{iter}_2 \leq T_2$}
\STATE Compute vectors $\mathbf{p}_i$ for $i \in \mathfrak{B}$
\STATE Compute vectors $\mathbf{q}_i$ for $i \in \mathfrak{B}$
\STATE Update $\boldsymbol{\theta}_e$ and $\boldsymbol{\theta}_M$ to minimize \eqref{eq4}
\STATE Update centers $\boldsymbol{\mu}^{(k)}, \, k=1,\ldots,\text{K}$ , using \eqref{eq2}
\ELSE
\STATE Compute  $\mathbf{q}_i$ for $i \in \mathfrak{B}$
\STATE Update $\boldsymbol{\theta}_e$ and $\boldsymbol{\theta}_M$ to minimize \eqref{eq5}
\ENDIF
\ENDFOR
\Statex{}
\Statex{\textbf{Final Cluster Assignments:}}
\STATE  Compute $\boldsymbol{q}_i$ for $\boldsymbol{x}_i$, $i=1,\ldots \text{N}$
\STATE Assign each data sample to the most probable cluster

\end{algorithmic}
\end{algorithm}

After $T_2$ epochs, when $\mathbf{q}$ becomes a relatively reliable space for identifying similar and dissimilar samples, we further train MNet using the loss function $\mathcal{L}^\prime_M$ defined in (\ref{eq5}). A pair contributes to $\mathcal{L}^\prime_M$ if its corresponding similarity value, in the $\mathbf{q}$ space, is not in the ambiguity region. As is demonstrated in Section \ref{experiments}, as the MNet training phase progresses, more and more pairs contribute to the training procedure. Again, the $\mathbf{u}$ space receives small updates through the backpropagation process when minimizing $\mathcal{L^\prime}_M$.
\begingroup\makeatletter\def\f@size{9}\check@mathfonts
\begin{align}\label{eq5}
    \mathcal{L}^\prime_M =\sum_{\mathbf{x}_i,\mathbf{x}_j\in \mathfrak{B}} \mathbbm{1}\{ \mathbf{q}_i^T \mathbf{q}_j \geq \zeta\}(1-\mathbf{q}_i^T \mathbf{q}_j)+\mathbbm{1}\{ \mathbf{q}_i^T \mathbf{q}_j \leq \gamma\}(\mathbf{q}_i^T \mathbf{q}_j)
\end{align}
Fig. \ref{Model}(b) shows the overall training procedure of the DCSS's second phase.

\subsection{Final cluster assignments}

To determine the final cluster assignment of a data point $\mathbf{x}_i$, we utilize the trained encoder and MNet networks to obtain the data representation in the $\mathbf{q}$ space, i.e. $\mathbf{q}_i$. $\mathbf{x}_i$ is assigned to the most probable cluster, i.e. the index corresponding to the highest element of $\mathbf{q}_i$ is the cluster label of $\mathbf{x}_i$. Such clustering assignment process is shown in Fig. \ref{Model}(c).

The pseudo-code of the DCSS algorithm is presented in Algorithm \ref{alg1}. 

\subsection{Proper choice of $\zeta$ and $\gamma$}

In this section, we discuss the selection of optimal hyperparameters $\zeta$ and $\gamma$ for our model. It is crucial to choose appropriate values for these hyperparameters, as the model's performance is greatly impacted by their values. When the value of $\zeta$ is high and $\gamma$ is small (e.g., $\zeta = 0.9$ and $\gamma = 0.1$), the model tends to consider only a small subset of the available data during the second training phase, resulting in the neglect of crucial information between true similar and dissimilar samples. Conversely, when the value of $\zeta$ is low and $\gamma$ is high (e.g., $\zeta = 0.5$ and $\gamma = 0.5$), the model may struggle to accurately distinguish between similar and dissimilar samples. Therefore, selecting optimal values for $\zeta$ and $\gamma$ is critical to ensure the effectiveness of our model.  \\ 
\textbf{Notation clarification:} Representation of the ith sample in the $\mathbf{q}$ space is shown by $\mathbf{q}_i$. The kth element of $\mathbf{q}_i$ is shown by $q_{ik}, k=1,\ldots, K$, where $K$ is the number of data clusters. Note that $\mathbf{q}_i$ is the MNet output when the input sample is $x_i$. Since we employ soft-max as the final layer of MNet, $ 0 \leq q_{il} \leq 1$ where $1 \leq l \leq K$ and the $\ell_1$-norm of $\mathbf{q}_i$ is equal to 1. Furthermore, as is discussed in the manuscript, parameters $\zeta$ and $\gamma$ are values between 0 and 1.

\begin{definition} \label{Def1}
Two data points, i.e. i and j, are adjacent (aka similar) if and only if $\mathbf{q}_i^T\mathbf{q}_j\geq \zeta$.
\end{definition}
\s{\brown{\textbf{Definition 1:} two data points, i.e. i and j, are adjacent if and only if $\mathbf{q}_i^T\mathbf{q}_j\geq \zeta$.}}

\begin{definition} \label{Def2}
Two data points, i.e. i and j, are in the same cluster if and only if the index of the maximum value in their corresponding \s{representation in the} $\mathbf{q}$ vector (i.e. $\mathbf{q}_i$ and $\mathbf{q}_j$) are equal. 
\end{definition}

\s{\brown{\textbf{Definition 2:} two data points, i.e. i and j, are in the same cluster if and only if the index of the maximum value in their corresponding \s{representation in the} $\mathbf{q}$ (i.e. $\mathbf{q}_i$ and $\mathbf{q}_j$) are equal.} \b{Note $||\mathbf{q}_k||_1 = 1$ for $k=1,\ldots,K$.}}

\begin{theorem} \label{theorem1}
Consider the ith and jth data points. \s{Considering two data points i and j,} Then \s{the maximum inner product of $\mathbf{q}_i$ and $\mathbf{q}_j$ is}: \s{less than the maximum elements of $\mathbf{q}_i$ and $\mathbf{q}_j$ or in other words:}
\begin{align}\label{eq77}
    \mathbf{q}_i^T\mathbf{q}_j \leq \min\big \{\max_l\{q_{il}\},\max_l\{q_{jl}\}\big\}
\end{align}
where $ \mathbf{q}_i^T\mathbf{q}_j$ is the inner product of the two vector $\mathbf{q}_i$ and $\mathbf{q}_j$. \s{$q_{il}$ and $q_{jl}$ denote the $lth$ element of $\mathbf{q}_i$ and $\mathbf{q}_j$ \s{ soft cluster assignment}, respectively.}
\end{theorem}

\begin{proof}
\s{\textbf{Proof 1:}}
Assume $\mathbf{q}_j^*$ is a maximal vector that satisfies the below inequality: 
\begin{align}\label{eqAss1}
    \mathbf{q}_i^T \mathbf{q}_j\leq \mathbf{q}_i^T \mathbf{q}_j^*,
\end{align}
where \s{for all \b{possible} $\mathbf{q}_j$, we have $\mathbf{q}_i^T \mathbf{q}_j\leq \mathbf{q}_i^T \mathbf{q}_j^*$ and} $||\mathbf{q}_j^*||_1 = 1$. In addition, assume the index of the maximum element of $\mathbf{q}_i$ is $r$:
\begin{align}\label{eqAss2}
   r=\argmax_l\{q_{il}\}
\end{align}
\s{$r=\argmax_l(q_{il})$.}

\noindent In the following, we first prove by contradiction that $\mathbf{q}_j^*$ must be a one-hot vector. Then we prove (\ref{eq77}).

\noindent \s{Proof by contradiction: Lets }Assume $\mathbf{q}_j^*$ is not a one-hot vector. Therefore, there exists at least one index\s{an index in $q_j^*$}, i.e. $e$, that its corresponding element $q_{je}^*$ is non-zero\s{has the following property}:
\begin{align}
\exists \,\,e: e\neq r \,\,\, \text{and}\,\,\, q_{je}^*\neq 0.
\end{align}
Now, let's define\s{consider} a vector $\hat{q}_j$ as follows\s{that its $lth$ element has the following property}:
\begin{align}
     \hat{q}_{jl}=\begin{cases}
               q_{je}^*+q_{jr}^*& \text{if} \,\,\,l=r\\
               0 & \text{if} \,\,\,l=e\\
               q_{jl}^* & O.W
            \end{cases},
\end{align}
where $\hat{q}_{jl}$ denotes the $l^{th}$ element of $\hat{\mathbf{q}}_j$.
Since $||\mathbf{q}_j^*||_1 = 1$, we can immediately show \s{conclude }that $||\hat{\mathbf{q}}_j||_1 = 1$. Moreover, we can represent the inner products\s{expand} $\mathbf{q}_i^T\mathbf{q}_j^*$ and $\mathbf{q}_i^T\hat{\mathbf{q}}_j$ as shown in \eqref{eq100} and \eqref{eq111}, respectively.

\begin{align}
    \mathbf{q}_i^T\mathbf{q}_j^*= q_{ir}q_{jr}^*+q_{ie}q_{je}^*+\sum_{l\neq r,e}q_{il}q_{jl}^* \label{eq100}\\
    \mathbf{q}_i^T\hat{\mathbf{q}}_j= q_{ir}(q_{jr}^*+q_{je}^*)+q_{ie}\times 0+\sum_{l\neq r,e}q_{il}q_{jl}^* \label{eq111}
\end{align}
Since $q_{ir}$ is the maximum element of $\mathbf{q}_i$, we can readily show\s{conclude} that $\mathbf{q}_i^T\mathbf{q}_j^*\leq \mathbf{q}_i^T\hat{\mathbf{q}}_j$, which contradicts the assumption shown in equation (\ref{eqAss1}); thus, $\mathbf{q}_j^*$ must be a one-hot vector. 
Therefore:
\begin{align} \label{eq144}
\mathbf{q}_i^T\mathbf{q}_j^* \leq \max_l \{q_{il}\}     
\end{align}
Considering (\ref{eq144}) and (\ref{eqAss1}), we have: 
\begin{align}\label{eq1555}
\mathbf{q}_i^T\mathbf{q}_j \leq \max_l\{q_{il}\}.
\end{align}

\noindent Similarly, for the ith sample, we can show that $\mathbf{q}_i^{*T}\mathbf{q}_j \leq \max_l \{q_{jl}\}$, hence:
\begin{align}\label{eq16666}
\mathbf{q}_i^T\mathbf{q}_j \leq \max_l\{q_{jl}\}.
\end{align}
(\ref{eq1555}) and (\ref{eq16666}) proves (\ref{eq77}).

\brown{\s{In this case, $\mathbf{q}_i^T\mathbf{q}_j \leq \max_l(q_{il})$. If we repeat the same process to find $p_i^*$, we will obtain $q_i^Tq_j \leq \max_l(q_{jl})$. Hence we can prove \eqref{eq77}.}} 
\end{proof}

\begin{corollary}\label{col1.1}
\b{For two adjacent samples, denoted as i and j, the maximum value among their $\mathbf{q}$ representations (i.e $\mathbf{q}_i$ and $\mathbf{q}_j$) are greater than $\zeta$.}
\end{corollary}
\begin{proof}
\b{This statement serves as a corollary derived from Theorem \ref{theorem1}, which is established through equations \eqref{eq77}, \eqref{eq1555}, and \eqref{eq16666}, along with the adjacency definition provided in Definition \ref{Def1}. In this context, we denote the indices of the maximum elements of $\mathbf{q}_i$ and $\mathbf{q}_j$ as $r$ and $o$ respectively, where $r = \argmax_l\{q_{il}\}$ and $o = \argmax_l\{q_{jl}\}$.}
\s{\brown{Moreover, from \eqref{eq77}\b{, \eqref{eq1555} and \eqref{eq16666}, and the adjacency definition provided in Definition 1}, we can conclude that if two samples i and j are adjacent, the maximum value among their elements is greater than $\zeta$. In other words:}}
\begin{align}\label{eq122}
    &\zeta \leq \mathbf{q}_i^T\mathbf{q}_j \leq min(q_{ir}, q_{jo}) \; \; \; \xrightarrow[]{}&  \begin{cases}
               \zeta \leq q_{ir}\\
               \zeta \leq q_{jo}
            \end{cases}
\end{align}

\end{proof}

\begin{figure}[h]
  \centering
  \includegraphics[width=0.9\linewidth]{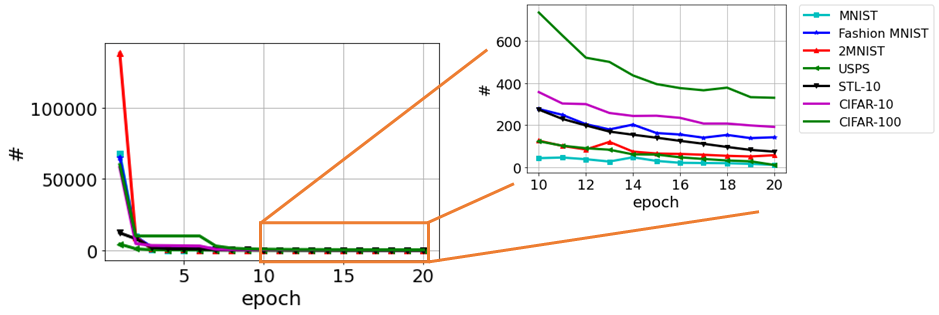}
  \caption{Number of data points that do not have any adjacent neighbors during the DCSS training in the second phase.\s{training of the second phase of DCSS.}}
  \label{adjacent_fig}
\end{figure}

\begin{corollary}\label{col3.1}
 If a data point has at least one adjacent neighbor, the maximum element of its corresponding  $\mathbf{q}$ is greater than $\zeta$.
\end{corollary}
\begin{proof}
We first empirically test the validity of the employed assumption, i.e. the existence of at least one adjacent (i.e. similar) sample for a data point,  on our datasets.
Fig. \ref{adjacent_fig} shows the number of data samples that are not similar to any other data points in the $\mathbf{q}$ space. As it can be seen, at the beginning of the second phase of DCSS, since MNet is initialized randomly, many data points do not have any adjacent neighbor (i.e. almost $\forall\,\,\, i,j ,  \; i\neq j :\,\, \mathbf{q}_i^T\mathbf{q}_j<\zeta$.). By minimizing (4) and (5) in the second phase of DCSS, similar samples are tightly packed in the $\mathbf{q}$ space; therefore, almost all samples have at least one adjacent neighbor. For example, only 0.5\% of the samples in the CIFAR-100 dataset have no adjacent sample by the end of the training phase.\s{there are only 330 data points, out of the 60,000 samples of the CIFAR-100 dataset, that are not similar to any other data samples at the end of the second phase.} Note that CIFAR-100 presents the worst case among the other datasets shown in Fig \ref{adjacent_fig}. All in all, we can roughly assume that each data point has at least one adjacent sample.

Let us consider an arbitrary data point, i, and one of its adjacent data points, j. From Corollary \ref{col1.1}, we can conclude that:
\begin{align}
     \begin{cases}
               \zeta \leq \max_l\{q_{il}\}\\
               \zeta \leq \max_l\{q_{jl}\}
            \end{cases}.
\end{align}
Thus, we proved that the maximum element of $\mathbf{q}_i$, where i is an arbitrary data point, is greater than $\zeta$.
\end{proof}

\begin{corollary}\label{col3.2}
Assume each data point has at least one adjacent neighbor and $\gamma<\zeta^2$. If two data points, i and k, are dissimilar, i and k are not from the same cluster.
 \begin{proof}
 We prove this corollary by contradiction where the contradiction assumption is: i and k are dissimilar, yet from the same cluster where $\gamma<\zeta^2$.\\ 
 Since i and k are in the same cluster, the index of the maximum element of $\mathbf{q}_i$ and $\mathbf{q}_k$ are the same. (i.e. $\beta = \argmax_l\{q_{il}\}=\argmax_l\{q_{kl}\}$).
 Since each data point has at least one adjacent neighbor, from Corollary \ref{col3.1}, we can conclude that:
 \begin{align}\label{eq27}
     \begin{cases}
               \zeta \leq q_{i\beta}\\
               \zeta \leq q_{k\beta}
            \end{cases}.
 \end{align}
 and we can represent $\mathbf{q}_i^T\mathbf{q}_k$ as follow:
 \begin{align}
     \mathbf{q}_i^T\mathbf{q}_k=q_{i\beta}q_{k\beta}+\sum_{l\neq \beta} q_{il}q_{kl}
 \end{align}
Therefore, $q_{i\beta}q_{k\beta}\leq \mathbf{q}_i^T\mathbf{q}_k$. Also, we know i and k are dissimilar. From \eqref{eq27}, we can conclude that:
\begin{align}\label{eq29}
    \zeta^2 \leq q_{i\beta}q_{k\beta}\leq \mathbf{q}_i^T\mathbf{q}_k \leq \gamma
\end{align}
\eqref{eq29} contradicts the assumption of $\gamma<\zeta^2$. Hence, i and k are in different clusters.
 \end{proof}
\end{corollary}

\begin{theorem} \label{theorem2}
For $\frac{2}{3}\leq \zeta$, if i and j are adjacent, they are in the same cluster -- i.e. r is equal to o where $r = \argmax_l\{q_{il}\}$ and $o = \argmax_l\{q_{jl}\}$.
\end{theorem}
\begin{proof}
First, we find an upper bound for $\mathbf{q}_i^T\mathbf{q}_j$, when the ith and jth samples are adjacent but from different clusters.

\noindent Since i and j are not from the same cluster, we can represent $\mathbf{q}_i^T\mathbf{q}_j$ as follows:
\begin{align}\label{eq12}
    \mathbf{q}_i^T\mathbf{q}_j &=  q_{ir}q_{jr} + q_{io}q_{jo} +\sum_{l\neq o,r} q_{il}q_{jl}\nonumber\\
    \xrightarrow[]{\forall l: \,\,\,q_{jl}\leq q_{jo}}&\leq q_{ir}q_{jr} + q_{io}q_{jo} +\sum_{l\neq o,r} q_{il}q_{jo}\nonumber \\
    &= q_{ir}q_{jr} +q_{jo}\Big(\sum_{l\neq r} q_{il}\Big)\nonumber\\ \xrightarrow[]{\sum_{l\neq r} q_{il} = 1- q_{ir}}&=q_{ir}q_{jr}+q_{jo}(1-q_{ir}) \nonumber\\
    \xrightarrow[]{q_{jr}\leq 1-q_{jo}}&\leq q_{ir}(1-q_{jo})+q_{jo}(1-q_{ir})\nonumber\\\xrightarrow[]{\text{from}\,\,\,\eqref{eq122}}& \leq q_{ir}(1-\zeta)+q_{jo}(1-\zeta)\nonumber \\\xrightarrow[]{ \{q_{ir},\,\, q_{jo} \} \in [0,1]\s{0 \leq q_{ir},\,\, q_{jo}\leq 1}}& \leq 2(1-\zeta).
\end{align}
Note that $q_{jr} \leq 1-q_{jo}$ because $\sum_{l\neq o} q_{jl}+q_{jr}+q_{jo}=1$. Note that all elements of a $\mathbf{q}$ vector are probability values between 0 and 1. 

\noindent Hence, as is shown \eqref{eq12}, if two samples are not from the same cluster, then the inner product of their corresponding $\mathbf{q}$ has an upper bound of $2(1-\zeta)$. Therefore, if two samples i and j are adjacent (see Definition 1) but from different clusters, then:
\begin{align}
    &\zeta \leq \mathbf{q}_i^T\mathbf{q}_j \leq 2(1-\zeta)\nonumber\\ \xrightarrow[]{}& \zeta \leq 2(1-\zeta)\xrightarrow[]{} \zeta \leq \frac{2}{3}.
    \end{align}
    
\noindent Thus, for $\frac{2}{3}< \zeta$, i and j cannot be from two different clusters. In other words, if two samples i and j are adjacent AND the user-settable parameter $\zeta$ is set to a value greater than $\frac{2}{3}$, then the two samples are from similar clusters, i.e., $r=o$. \s{must be in the same cluster.  \r{YOU MUST PROVE IT.} Otherwise, their inner product will not be greater than $\zeta$.} In this paper we set $ \zeta=0.8 > \frac{2}{3}$.
\end{proof}

\begin{corollary}
Assume $\zeta>\frac{2}{3}$. Consider three data points: i, j, and k. If i and j, and also i and k are adjacent, then j and k are from the same cluster.
\end{corollary}
\begin{proof}
Since i and j (i and k) are adjacent and $\zeta>\frac{2}{3}$, from Theorem \ref{theorem2}, we can conclude that i and j (i and k) are in the same cluster; hence, the three samples i, j, and k all are in the same cluster.
\end{proof}
 
\s{\red{MOHAMMDREZA: YOU SHOULD CLARIFY THE BELOW PARAGRAPHS. IT IS NOT CLEAR WHAT YOU MEAN. \b{I DID IT.}}} 

\begin{theorem} \label{theorem3}
Consider three data points i, j, and k where i and j also i and k are adjacent (aka similar). Assume $\zeta > \frac{2}{3}$. If $\gamma < \zeta^2$, then the two samples j and k are not dissimilar (i.e. $\mathbf{q}_j^T\mathbf{q}_k \nless \gamma$)\s{; therefore, their corresponding pair does not contribute as a dissimilar pair in minimizing the loss function defined in equation \eqref{eq5})}.
\end{theorem}
\begin{proof}
Considering Theorem \ref{theorem2}, i and j and k are in the same cluster. Therefore, we do not want to include the pair of j and k samples as a dissimilar pair when minimizing the loss function defined in equation (5) of the original manuscript.
\brown{\s{In order to not include j and k as a negative pair in the loss function, we should find a lower bound for $q_j^T\mathbf{q}_k$.}}
\s{\brown{Because $\zeta \leq q_i^Tq_j$ and $\zeta \leq q_i^T\mathbf{q}_k$, we can infer \eqref{eq15} from \eqref{eq122}.  }
\begin{align}\label{eq15}
     \begin{cases}
               \zeta \leq \max_l\{q_{jl}\}\\
               \zeta \leq \max_l\{q_{kl}\}
            \end{cases}
\end{align}}

\noindent Since j and k are in the same cluster and knowing that the index of the maximum element in a $\mathbf{q}$ vector shows the cluster of the corresponding sample, we have: 
\begin{align}
    \eta = \argmax_l\{q_{jl}\} = \argmax_l\{q_{kl}\}.
\end{align}
Thus,
\begin{align}
    \mathbf{q}_j^T\mathbf{q}_k= \sum_{l\neq \eta} q_{jl}q_{kl}+q_{j\eta}q_{k\eta}.
\end{align}
Therefore,
\begin{align} \label{eq2222}
\mathbf{q}_j^T\mathbf{q}_k \geq q_{j\eta}q_{k\eta}.
\end{align}

\noindent Since $\zeta \leq \mathbf{q}_i^T\mathbf{q}_j$ and $\zeta \leq \mathbf{q}_i^T\mathbf{q}_k$, we can infer \eqref{eq15} from \eqref{eq122}. 
\begin{align}\label{eq15}
     \begin{cases}
               \zeta \leq q_{j\eta}\\
               \zeta \leq q_{k\eta}
            \end{cases}.
\end{align}
\\
From \eqref{eq2222} and \eqref{eq15}, we can obtain below:
\begin{align}
    \zeta^2 \leq q_{j\eta}q_{k\eta} \leq \mathbf{q}_j^T\mathbf{q}_k.
\end{align}
\s{\brown{Moreover, from Theorem \ref{theorem2}, we know that samples j and k are in the same cluster; therefore, we have:}

\begin{align}
    \zeta^2 \leq \max_l(q_{jl})\max_l(q_{kl}) \leq q_j^T\mathbf{q}_k
\end{align}}
\\
Thus, if we choose $\gamma<\zeta^2$, we will not include the pair of samples j and k as a dissimilar pair\s{two negative pairs} in equation (5) of the main manuscript. 
In this paper, $\gamma$ is set to 0.2, i.e. $\gamma=0.2<0.8^2$.
\end{proof}

\section{Experiments}\label{experiments}
In this section, the effectiveness of our proposed DCSS framework is demonstrated on eight benchmark datasets through conducting a rigorous set of experiments. The DCSS clustering performance on the eight benchmark datasets is compared with seventeen \s{conventional and state-of-the-art deep-learning-based}clustering methods. 

\subsection{Datasets}\label{dset}

The effectiveness of the proposed method is shown on eight widely used datasets. Considering the unsupervised nature of the clustering task, we concatenate training and test sets when applicable. Combining train and test datasets is a common practice in the clustering research field \cite{dec,dkm,idec,dcn,dsc}.
The datasets are:
\\
(1) MNIST \cite{mnist} consists of 60,000 training and 10,000 test gray-scale handwritten images with size $28\times 28$. This dataset has ten classes, i.e. $K=10$.
\\
(2) Fashion MNIST \cite{fashion_mnist} has the same image size and number of samples as MNIST. However, instead of handwritten images, it consists of different types of fashion products. This makes it fairly more complicated for data clustering compared to the MNIST dataset. It has ten classes of data, i.e. $K=10$.
\\
(3) 2MNIST is a more challenging dataset created through the concatenation of the two MNIST and Fashion MNIST datasets. Thus, it has 140,000 gray-scale images from 20 classes, i.e. $K=20$.
\\
(4) USPS \cite{usps} contains of 9,298 $16\times 16$ handwritten images from the USPS postal service. It contains ten classes of data, i.e. $K=10$.
\\
(5) CIFAR-10 \cite{cifar} is comprised of 60,000 RGB images of 10 different items (i.e. $K=10$), where the size of each image is $32\times 32$. 
\\
(6) STL-10 \cite{stl} is a 10-class image recognition dataset comprising of 13,000 $96\times 96$ RGB images. The number of clusters $K$ for this dataset is set to 10.
\\
(7) CIFAR-100 \cite{cifar} is similar to the CIFAR-10, except it has 20 super groups based on the similarity between images instead of 10 classes. The number of clusters $K$ for this dataset is set to 20.
\\
(8) 
ImageNet-10 comprises ten classes from the larger ImageNet dataset \cite{ILSVRC15}. Despite its reduced size, ImageNet-10 preserves the structural and thematic diversity of the original ImageNet dataset.
\\
\\
The network architecture and implementation details of DCSS for each dataset are presented in Section 1 and Section 2 of the supplementary material file, respectively.   

\s{\subsection{Networks Architecture}

\b{The proposed DCSS method includes an autoencoder and a fully connected MNet. This section presents the structure of these networks.}
We use two variations of autoencoders\b{, depending on the dataset nature (i.e. RGB or gray-scale),} when training the proposed DCSS framework. 

\b{For gray-scale datasets, we propose to use an asymmetric autoencoder; where, following \cite{resnet}, we propose to use the bottleneck layer shown in Fig. \ref{AE_image}(c) in the encoder structure. Fig. \ref{AE_image}(a) and (b) respectively show the encoder and decoder structure\b{s} of the proposed asymmetric AE. Employing such an asymmetric structure provides a more discriminative latent space.} Hyperparameters of the proposed AE for each dataset \b{are} indicated in \b{Section} \ref{imp_detail} 

\b{For the RGB datasets, we first apply a ResNet-152, pre-trained on ImageNet \cite{deng2009imagenet}, to extract abstract features. Then, we feed the extracted features to a \brown{symmetric} fully connected AE. \brown{Inspired by \cite{dec}, we set the \b{AE architecture} to 2048-500-500-2000-d for RGB datasets\b{, where} Relu activation function is utilized in all layers.}}  

\brown{ MNet is a fully connected network that takes the \b{d dimensional} latent space of the AE ($\mathbf{u}$ space) as input and generates a K dimensional output $\mathbf{q}$. The architecture of MNet is d-128-128-128-K for all datasets except CIFAR-100. Since CIFAR-100 is a more complicated dataset, it needs a more complex MNet architecture\b{; so we set the MNet architecture for CIFAR-100 to} d-1000-1000-1000-K. Batch normalization and Relu activation functions are utilized for all datasets in all layers of MNet except the last layer in which we used the soft-max function.}

\subsection{Implementation Details}\label{imp_detail}

\brown{In this section, we discuss hyperparameter \b{values and implementation details} of DCSS. }

\brown{Network’s hyperparameters $n,\, p_1,\, s_1,\, p_2,\, s_2, \,p_3,\, s_3,\, f_1,$ and $f_2$ (shown in Fig. \ref{AE_image} ) are respectively set to 28, 2, 2, 1, 2, 2, 2, 5, and 4 for MNIST, Fashion MNIST, and 2MNIST; these parameters are set to  16, 1, 1, 2, 2, 0, 1, 4, and 5 for the USPS dataset. The latent space dimension $d$ is set to 10 for gray-scale images and 20 for RGB images.}

\b{Following \cite{,dkm,dcn, idec, dec}, in order to initialize the parameters of DCSS$_u$'s, i.e. $\theta_e$, $\theta_d$ and $\mu^{(k)}$ for $k=1,\ldots,K$, we train an autoencoder where the end-to-end training is performed by only minimizing the samples reconstruction losses. Adam optimization method \cite{adam}, with the same parameters mentioned in the original paper, is used for training. $\theta_e$, $\theta_d$ are then initialized with the parameters of the trained autoencoder's parameters. We apply the k-means algorithm \cite{kmeans} to the latent space of the trained autoencoder and initialize $\mu^{(k)}, k=1, \ldots, K$ to the centers defined by k-means.}

\brown{\b{For all datasets, in} the first phase of DCSS, $\alpha$, $\text{Maxiter}_1$, $T_1$, and $m$ are respectively set to 0.1, 200, 2, and 1.5.}
\brown{The second phase hyperparameters $\zeta$, $\gamma$, $T_2$, and $\text{MaxIter}_2$ are respectively set to 0.8, 0.2, 5, and 20. We utilize Adam optimizer for updating weights of the AE and MNet, and their learning rates are set to $10^{-5}$ and $10^{-3}$, respectively. } 

\b{All algorithms were implemented in Python using the PyTorch framework. All codes are run on \b{Google Colaboratory} GPU \r{(Tesla K80) with 12GB RAM}. \b{The proposed algorithm codes are included} in the supplementary materials.}}

\subsection{Evaluation Metrics}\label{Metrics}
We utilize two standard metrics to evaluate clustering performance, including clustering accuracy (ACC) \cite{acc} and normalized mutual information (NMI) \cite{nmi}. ACC finds the best mapping between the true and predicted cluster labels. NMI finds a normalized measure of similarity between two different labels of the same data point. The ACC and NMI formulations are shown below:

\begin{subequations} \label{eq6}
\begin{eqnarray}
 & ACC = \max_{map}\frac{\sum_{i=1}^{N}\mathbbm{1}\{l_i=map(c_i)\}}{N} \\
 & NMI = \frac{I(l;c)}{max\{H(l),H(c)\}}
\end{eqnarray}
\end{subequations}
\\
where $l_i$ and $c_i$ denote the true and predicted labels for the data point $\mathbf{x}_i$. $map(.)$ indicates the best mapping between the predicted and true labels of data points. $I(\mathbf{l};\mathbf{c})$ denotes the mutual information between true labels $\mathbf{l}=\{l_1,l_2,...,l_N\}$ and predicted cluster assignments $\mathbf{c}=\{c_1,c_2,...,c_N\}$ for all data points. $H(.)$ presents the entropy function. ACC and NMI range in the interval [0,1], where higher scores indicate higher clustering performance.

\subsection{Clustering Performance} \label{ClusteringPerformance}

\begin{table*}[h]
  \caption{ACC and NMI on the benchmark datasets for different clustering methods. The second best result is shown by *.}
  \label{table1}
  \centering
  \scalebox{0.50}{
    \begin{tabular}{|c||*{9}{c|c|}}\hline
      \backslashbox{Method}{Datasets} & \multicolumn{2}{c|}{MNIST} & \multicolumn{2}{c|}{Fashion MNIST} & \multicolumn{2}{c|}{2MNIST} & \multicolumn{2}{c|}{USPS} & \multicolumn{2}{c|}{CIFAR-10} & \multicolumn{2}{c|}{STL-10} & \multicolumn{2}{c|}{CIFAR-100} & \multicolumn{2}{c|}{ImageNet-10}\\ \hline\hline
      & ACC & NMI & ACC & NMI & ACC & NMI & ACC & NMI & ACC & NMI & ACC & NMI & ACC & NMI & ACC & NMI \\ \hline
      k-means & 53.20 & 50.00 & 47.40 & 51.20 & 32.31 & 44.00 & 65.67 & 62.00 & 22.90 & 8.70 & 19.20 & 12.50 & 13.00 & 8.40 & 24.1 & 11.90  \\ \hline
      LSSC & 71.40 & 70.60 & 49.60 & 49.70 & 39.77 & 51.22 & 63.14 & 58.94 & 21.14 & 10.89 & 18.75 & 11.68 & 14.60 & 7.92 & - & -\\ \hline
      LPMF & 47.10 & 45.20 & 43.40 & 42.50 & 34.68 & 38.69 & 60.82 & 54.47 & 19.10 & 8.10 & 18.00 & 9.60 & 11.80 & 7.90 & - & - \\ \hline
      DEC & 84.30 & 83.72 & 51.80 & 54.63 & 41.20 & 53.12 & 75.81 & 76.91 & 30.10 & 25.70 & 35.90 & 27.60 & 18.50 & 13.60  & 38.10 & 28.20 \\ \hline
      IDEC & 88.13 & 83.81 & 52.90 & 55.70 & 40.42 & 53.56 & 75.86 & 77.68 & 36.99 & 32.53 & 32.53 & 18.85 & 19.61 & 14.58 & 39.40 & 29.00\\ \hline
      DCN & 83.00 & 81.00 & 51.22 & 55.47 & 41.35 & 46.89 & 73.00 & 71.90 & 30.47 & 24.58 & 33.84 & 24.12 & 20.17 & 12.54 & 37.40 & 27.30\\ \hline
      DKM & 84.00 & 81.54 & 51.31 & 55.57 & 41.75 & 46.58 & 75.70 & 77.60 & 35.26 & 26.12 & 32.61 & 29.12 & 18.14 & 12.30 & 38.10 & 29.80\\ \hline
      VaDE & 94.50 & 87.60 & 50.39 & 59.63 & 40.35 & 58.37 & 56.60 & 51.20 & 29.10 & 24.50 & 28.10 & 20.00 & 15.20 & 10.80 & 33.40 & 19.30\\ \hline
      DAC & 97.75*& 93.51* & 62.80&58.90 & - & - & -& - & 52.18&39.59& 46.99 & 36.56 & 23.75 & 18.52 & 52.70 & 39.40\\ \hline 
      GANMM & 64.00 & 61.00 & 34.00 & 27.00 & - & - & 50.12 & 49.35 & - & - & - & - & - & - & - & -\\ \hline
      CC & 88.56 & 84.21 & 64.52 & 61.45 & 42.15 & 58.89 & 81.21 & 79.45 & 77.00 & 67.80 & 85.00 & 76.40 & 42.30 & 42.10 & 89.30 & 85.00\\ \hline
      PICA & - & - & - & - & - & - & - & - & 69.60 & 59.10 & 71.30 & 61.10 & 33.70 & 31.00 & 87.00 & 80.20\\ \hline
      EDESC & 91.30 & 86.20 & 63.10* & 67.00* & - & - & - & - & 62.70 & 46.40 & 74.50 & 68.70 & 38.50 & 37.00 &-&-\\ \hline
      IDFD & - & - & - & - & - & - & - & - & 81.50 & 71.10 & 75.60 & 64.30 & 42.50 & 42.60 & 95.40* & 89.80*\\ \hline
      MICE & - & - & - & - & - & - & - & - & 83.50* & 73.70* & 75.20 & 63.50 & 44.00 & 43.60 & - & -\\ \hline 
     DCCM & - & - & - & - & - & - & 68.60 & 67.50 & 62.30& 49.60 & 48.20 & 37.60 & 32.70 & 28.50 &  71.00 & 60.80 \\ 
      \hline \hline
      AE + k-means & 86.03 & 80.25 & 57.94 & 57.15 & 44.01 & 62.80 & 75.11 & 74.45 & 80.23 & 68.65 & 85.29 & 76.00 & 43.81 & 42.86 & 87.23 & 85.22 \\  \hline
      DCSS$_u$ & 95.99 & 89.95 & 62.90 & 63.58 & 45.31* & 63.00* & 82.93* & 81.84* & 82.43 & 71.32 & 86.47* & 76.52* & 44.08* & 43.70* & 89.46 & 86.30\\ \hline
      DCSS & \textbf{98.00} & \textbf{94.71} & \textbf{66.40} & \textbf{67.10} & \textbf{48.57} & \textbf{67.80} & \textbf{87.21} & \textbf{86.10} & \textbf{84.16} & \textbf{75.09} & \textbf{87.91} & \textbf{77.39} & \textbf{45.10} & \textbf{44.51} & \textbf{95.70} & \textbf{90.52}\\ \hline
    \end{tabular}
  }
\end{table*}

The effectiveness of our proposed DCSS method is compared against seventeen well-known algorithms, including conventional and state-of-the-art deep-learning-based clustering methods, using the commonly used evaluation metrics ACC and NMI, defined in Section \ref{Metrics}.

The conventional clustering methods are k-means \cite{kmeans}, large-scale spectral clustering (LSSC) \cite{lssc}, and locality preserving non-negative matrix factorization (LPMF) \cite{lpmf}. Deeplearning-based algorithms are deep embedding clustering (DEC) \cite{dec}, improved deep embedding clustering (IDEC) \cite{idec}, deep clustering network (DCN) \cite{dcn}, deep k-means (DKM) \cite{dkm}, variational deep embedding (VaDE) \cite{vade}, GAN mixture model for clustering (GANMM) \cite{ganmm}, deep adaptive clustering (DAC) \cite{dacc}, and the very recent clustering methods such as contrastive clustering (CC) \cite{cc}, deep semantic clustering by partition confidence maximization (PICA) \cite{pica}, efficient deep embedded subspace clustering (EDESC) \cite{edesc}, instance discrimination and feature decorrelation (IDFD) \cite{idfd}, Mixture of contrastive experts for unsupervised image clustering (MICE) \cite{mice}, and deep comprehensive correlation mining (DCCM) \cite{dccm}. 
In addition, we report the clustering performance of a baseline method AE + k-means in which k-means is simply applied to the latent representation of an AE that has a similar architecture as the AE used in the DCSS method, trained based on minimizing the dataset reconstruction loss. More details about the comparing algorithms can be found in Section \ref{RL work}.

We also demonstrate the success of the first phase of DCSS, presented in Section \ref{step1}, in creating the reliable subspace $\mathbf{u}$ in which the data points form hypersphere-like clusters\s{are effectively gathered} around their corresponding cluster center. To this end, we only implement the first phase of the DCSS algorithm -- i.e., we train the DCSS's AE through minimizing the loss function presented in (\ref{new_eq1}), where the AE architecture and its initialization are similar to those presented in Section 1 of the supplementary material file. After training the $\mathbf{u}$ space, we perform a crisp cluster assignment by considering each data hypersphere-like group in the $\mathbf{u}$ space as a data cluster and assigning each data point to the one with the closest center. In the following tables and figures, clustering using only the first phase is shown as DCSS$_u$. A preliminary version of DCSS$_u$ is presented in \cite{ijcnn}. 

The clustering performance of DCSS$_u$ and DCSS, along with the comparison algorithms, are shown in Table \ref{table1}. For the comparison methods, if the ACC and NMI of a dataset are not reported in the corresponding original paper, we ran the released code with the same hyper-parameters discussed in the original paper. When the code is not publicly available or not applicable to the dataset, we put dash marks (-) instead of the corresponding
results. The best result for each dataset is shown in bold. The second top results are shown with *.

Several observations can be made from Table \ref{table1}: (1) The proposed DCSS method outperforms all of our comparison methods on all datasets. (2) The first phase of DCSS (shown as DCSS$_u$) effectively groups the data points around their corresponding centers. This can be inferred from DCSS$_u$'s ACC and NMI values. When measuring ACC, DCSS$_u$ outperforms other methods in five out of seven datasets, and it exhibits even stronger performance by outperforming competitors in six out of seven datasets when considering NMI.\s{Indeed, DCSS$_u$ outperforms all the comparison clustering methods except DSC which is one of the very most recent state-of-the-art AE-based clustering methods. DCSS$_u$ outperforms DSC in 4 out of the seven datasets and provides competitive results on the remaining three ones.} (3) Effectiveness of the self-supervision with similar and dissimilar pairs of samples can be inferred by comparing DCSS with DCSS$_u$. It can be seen that DCSS significantly outperforms DCSS$_u$ on all datasets. (4) Effectiveness of the AE's loss function proposed in equation (\ref{new_eq1}) compared to the case of training AE with only the reconstruction loss can be inferred by comparing the DCSS$_u$ performance with the baseline method AE+k-means. As can be seen, the DCSS$_u$ clearly outperforms AE+k-means on all datasets.

\subsection{t-SNE visualization} \label{t-SNE visualization}

\begin{figure*}[h]
  \centering
  \includegraphics[width=0.80\linewidth]{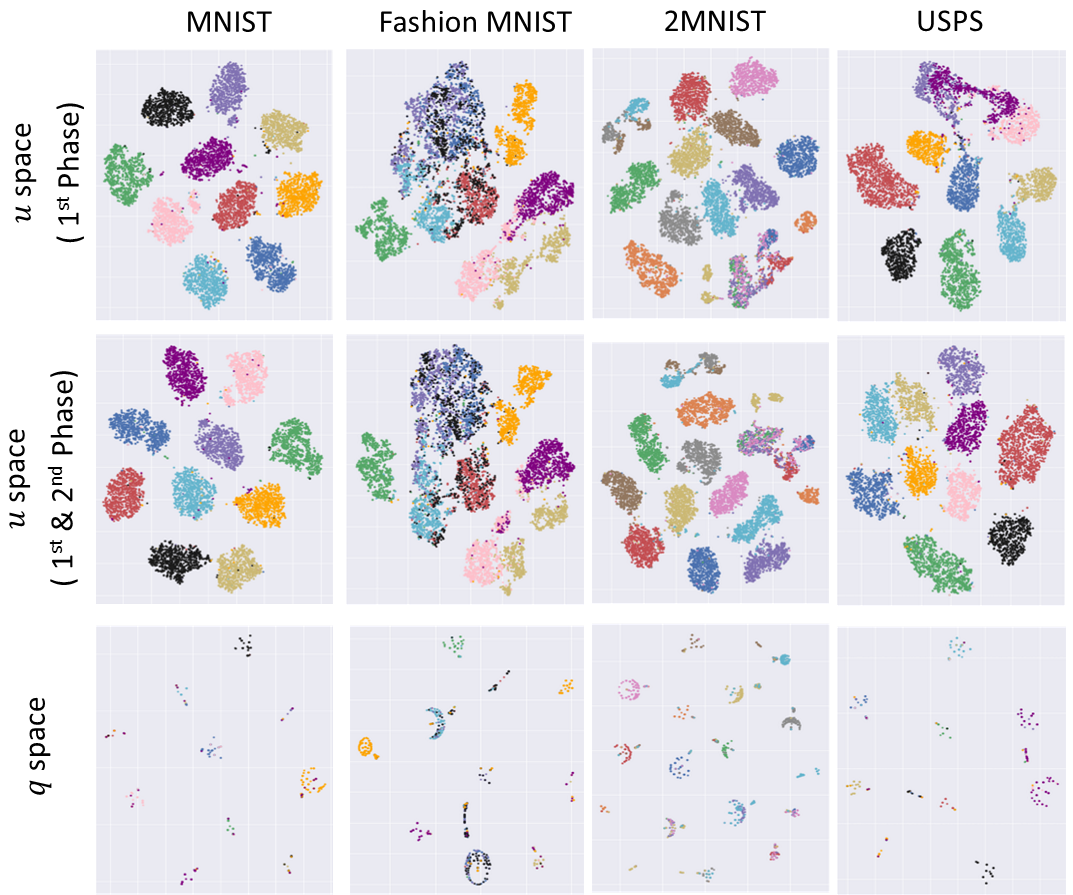}
  \caption{Clustering visualization of different phases of DCSS using t-SNE for different benchmark datasets. For reference, the visualization for the baseline model AE+k-means is shown in the first row.\s{The \b{second} \r{Mohammadreza: in both text and caption, you are referring to wrong row numbers.   Explanation of the first row AE+kemans is also missed! Fix them all.} The first row shows the grouping result of the first phase of DCSS, where a sample sits close to its corresponding group center by minimizing weighted reconstruction and centering losses. The final $\mathbf{u}$ space of DCSS, obtained by completing the first and second phases of DCSS, is shown in the second row. In the second phase, DCSS aims to refine the $\mathbf{u}$ space (obtained in the first phase) and train the $\mathbf{q}$ space by employing pairwise similarity between data points. The last row depicts the final output of MNet ($\mathbf{q}$ space) for all data points.} Axes range from -100 to 100.}
  
\label{TSNE}
\end{figure*}

Fig. \ref{TSNE} illustrates the effectiveness of different phases of our proposed DCSS framework for all datasets, where t-SNE \cite{tsne} is used to map the output of DCSS's encoder and MNet to a 2D space. The different colors correspond to the different data clusters. 

The first row of Fig. \ref{TSNE} shows the representation of different data points in the $\mathbf{u}$ space, i.e. the latent space of the DCSS's AE, only after completing the first phase discussed in Section \ref{step1}. As it can be seen, after completing the first phase of DCSS, different clusters of data points are fairly separated, sit near their corresponding centers, and form spheres; however, not all clusters are well separated. For example, in the USPS dataset, the data clusters shown in pink, purple, and magenta are mixed together. This indicates the insufficiency of the reconstruction and centering losses for the clustering task.

The second row of Fig. \ref{TSNE} shows the data representations in the $\mathbf{u}$ space after completing the second phase of DCSS discussed in Section \ref{2&3}, where $\mathbf{u}$ is refined by minimizing \eqref{eq4} and \eqref{eq5}. As it can be seen, refining the $\mathbf{u}$ space employing pairwise similarities results in more \b{dense} and separate cluster distributions.
For example, the pink, purple, and magenta clusters of USPS are now well distinguishable in the new refined $\mathbf{u}$ space. As another example, see samples of the three clusters shown in red, olive, and brown of the 2MNIST dataset. These clusters are more separable in the refined $\mathbf{u}$ space compared to the corresponding representation shown in the first row.

The last row in Fig. \ref{TSNE} depicts the output space of MNet (i.e. the $\mathbf{q}$ space), in which we make decisions about final cluster assignments of data points. As is expected, clusters in this space have low within- and high between-cluster distances, and cluster distributions can take non-hypersphere patterns. As an example, consider the cyan and the purple clusters of the Fashion MNIST. These clusters are mixed in the $\mathbf{u}$ space, but they are completely isolated in the $\mathbf{q}$ space.

\subsection{Loss function convergence}
\begin{figure*}[t]
  \centering
  \includegraphics[width=1\linewidth]{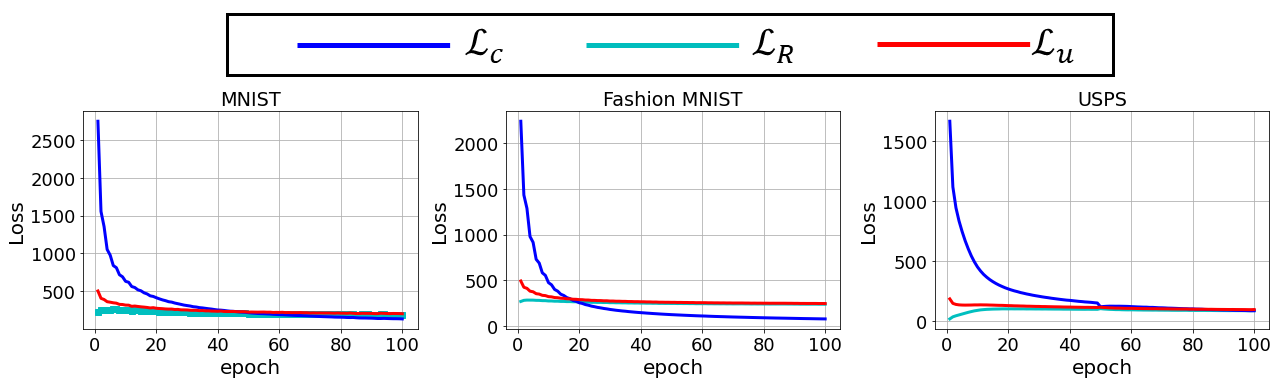}
  \caption{The reconstruction loss $\mathcal{L}_r$, centering loss $\mathcal{L}_c$, and total loss $\mathcal{L}_u$ of the first phase of DCSS vs. training epochs, for different datasets.\s{ At the earlier epochs, the latent representation of the data points are irregularly scattered around the group centers, which causes the high centering loss value. During the first step of DCSS, our proposed algorithm aims to simultaneously minimize centering and reconstruction losses; this leads to an increase in the reconstruction loss (and a decrease in the centering loss), in the later training epochs. When the first phase is complete, the latent representation of data points sit close to the group centers, which causes convergence of centering loss to a small value. }}
\label{loss_first}
\end{figure*}

\begin{figure}[t]
  \centering
  \includegraphics[width=0.5\linewidth]{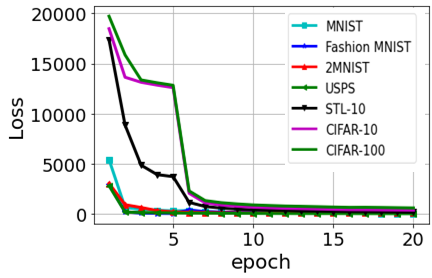}
  \caption{The second phase's loss function of DCSS for different benchmark datasets. \s{In the first $T_2$ epochs, when MNet is na\"ive, we make use of the $\mathbf{u}$ space to measure pairwise similarities; hence, the first $T_2=5$ epochs show the loss value defined in \eqref{eq4}. After the first $T_2$ epochs, all datasets show an eye-catching decrease in the loss values since we switch from using the $\mathbf{u}$ space to using the more reliable space $\mathbf{q}$ when measuring pairwise similarities by minimizing the loss defined in \eqref{eq5}.}  }  
\label{loss_second}
\end{figure}

Fig. \ref{loss_first}. depicts the average, over different clusters on different batches of data points, of the reconstruction, centering, and total losses corresponding to the first phase of DCSS (i.e. DCSS$_u$) shown in (\ref{new_eq1}). As can be seen, all losses are converged at the end of training. The noticeable reduction in the centering loss shows the effectiveness of our proposed approach in creating a reliable $\mathbf{u}$ space in which the data points are gathered around the centers.\s{training procedure in gathering data points near \o{group} centers.} Moreover, the figures show that at the first training epochs, our method trades the reconstruction loss for improved centering performance. This proves the insufficiency of the reconstruction loss in creating a reliable latent space for data clustering. 

In Fig. \ref{loss_second}, we investigate the convergence of the second phase losses, shown in equations (\ref{eq4}) and (\ref{eq5}). Since we initialize the MNet randomly, at the first few epochs, MNet has little knowledge about the lower-dimension representation of the data points in the $\mathbf{q}$ space; thus, we face a high loss value. As the training process progresses, the loss value drops and converges to zero at the end of the training process. In the first $T_2$ epochs ($T_2=5$), the algorithm minimizes the loss presented in (4). It minimizes (5) in the remaining epochs. The continuity of the loss reduction over epochs, along with the sharp loss drop at the 5th epoch, confirms the effectiveness of our proposed strategy in employing $\mathbf{u}$ for similarity measurements in the early epochs and then $\mathbf{q}$ in the later epochs.

%


%

\subsection{Performance on Imbalanced Dataset}

\begin{figure}[h]
  \centering
  \includegraphics[width=1\linewidth]{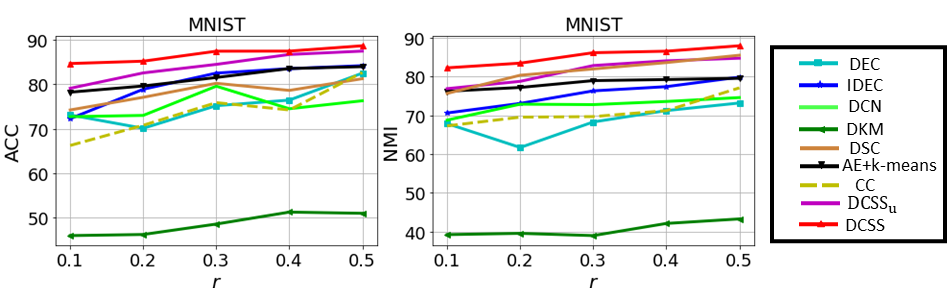}
  \caption{The clustering performance of different methods on imbalanced samples of MNIST.}
\label{imbalanced}
\end{figure}

To demonstrate the effectiveness of our proposed DCSS method on an imbalanced dataset, we randomly collect five subsets of the MNIST dataset with different retention rates $r\in\{0.1,0.2,0.3,0.4,0.5\}$, where samples of the first class are chosen with the probability of $r$ and the last class with probability of 1, with the other classes linearly in between. Hence, on average, the number of samples for the first cluster is $r$ times less than that of the last cluster. As is shown in Fig. \ref{imbalanced}, our proposed DCSS framework significantly outperforms our comparison methods for all r values. This indicates the robustness of DCSS on imbalanced data. As is expected, in general, for all methods, increasing $r$ results in a higher performance because the dataset gets closer to a balanced one. Higher performance of DCSS on imbalanced datasets can be associated with two factors: (1) considering an individual loss for every cluster in the 1st phase, and (2) considering the pairwise data relations.

\subsection{Visualization of $\mathbf{q}$ vectors}
\begin{figure*}[t]
  \centering
  \includegraphics[width=0.9\linewidth]{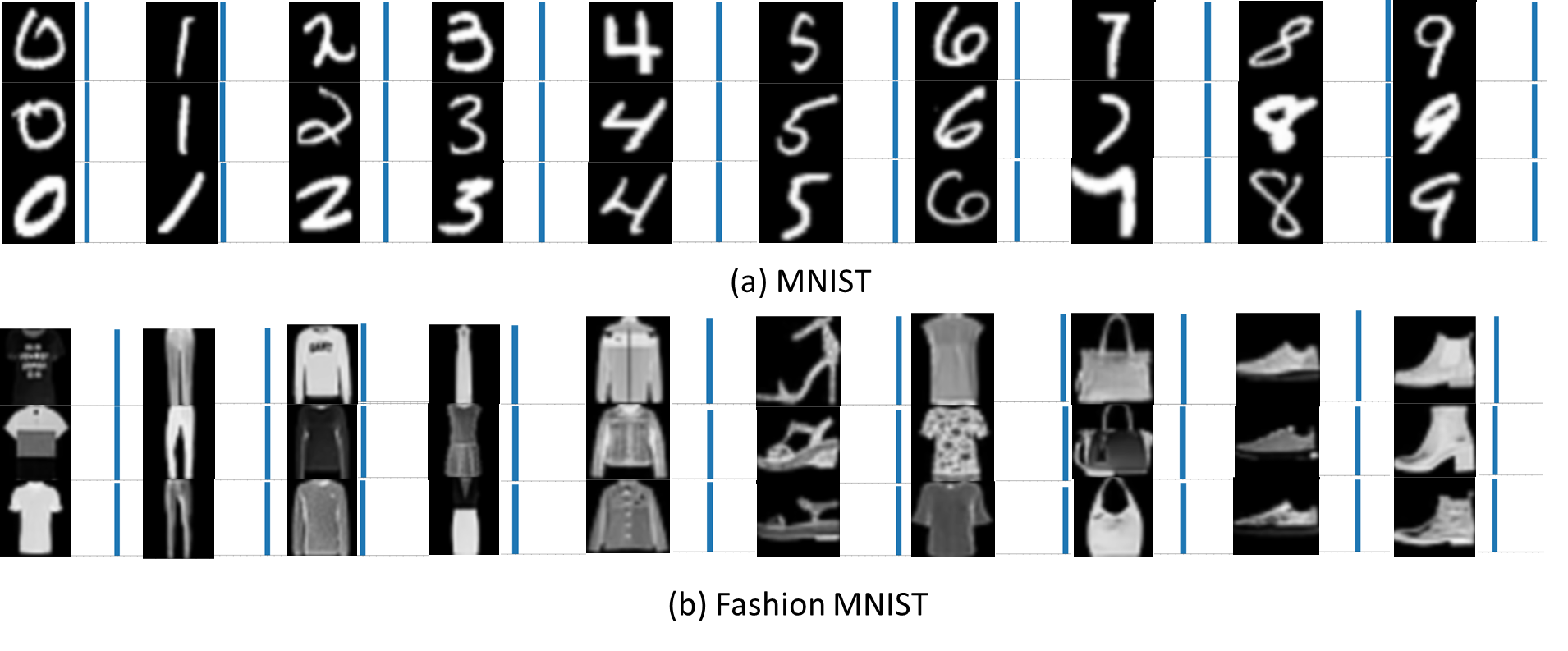}
  \caption{Visualization of $\mathbf{q}$ for samples from (a) MNIST and (b) Fashion MNIST datasets. The $\mathbf{q}$ vector for each image is depicted beside the image.  The vertical axes range from 0 to 1.}
\label{dataset_visualization}
\end{figure*}

\begin{figure}[h]
  \centering
  \includegraphics[width=1\linewidth]{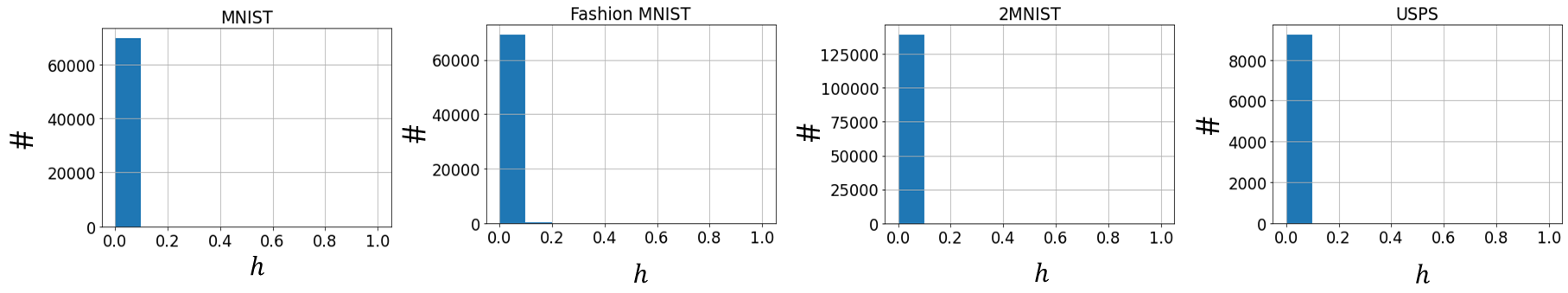}
  \caption{Histogram plot of $h_i= ||\mathbf{I}_i-\mathbf{q}_i||_1, \, i=1,\ldots, N$ where $\mathbf{I}_i$ is the one-hot crisp assignment corresponding to $\mathbf{q}_i$.}

\label{histogram}
\end{figure}

Fig. \ref{dataset_visualization} shows the representations of the data points from various clusters in the $\mathbf{q}$ space. As can be seen, the proposed DCSS method results in representations that are very close to the one-hot vectors. Note that the kth element of $\mathbf{q}_i$ denotes the probability of sample $\mathbf{x}_i$ being in the kth data cluster. The closer $\mathbf{q}_i$ is to the one-hot vector, the more confidently a crisp cluster assignment can be made. As is proved in Corollary 1.2 of Section 3.4, if data point $\mathbf{x}_i$ has at least one similar neighbor, the maximum element of $\mathbf{q}_i$ is greater than $\zeta$. In our experiments, $\zeta$ is set to 0.8. This can justify the aggregation of the data points near the one-hot vectors in the $\mathbf{q}$ space.

To further demonstrate convergence of the $\mathbf{q}$ representations to one-hot vectors, histogram of residuals $h_i= ||\mathbf{I}_i-\mathbf{q}_i||_1, \, i=1,\ldots, N$ for all datasets are shown in Fig. \ref{histogram}; where $||.||_1$ indicates the $\ell_1$-norm and $\mathbf{I}_i$ is the one-hot crisp assignment corresponding to $\mathbf{q}_i$ -- i.e. the index of the non-zero element of $\mathbf{I}_i$ is equal to the index of the maximum element of $\mathbf{q}_i$.
As can be seen in Fig. \ref{histogram}, the representation of almost all data points in the $\mathbf{q}$ space is very close to their corresponding one-hot vector. 

\subsection{\b{Ablation Studies}}
\subsubsection{Effectiveness of the First Phase}
\begin{table*}[h]
  \caption{ACC and NMI on the benchmark datasets for different methods.} \label{newtable}
  \centering
  \scalebox{0.6}{
    \begin{tabular}{|c||*{7}{c|c|}}\hline
      \backslashbox[3em]{Method}{Datasets}
      & \multicolumn{2}{c|}{MNIST} & \multicolumn{2}{c|}{Fashion MNIST} & \multicolumn{2}{c|}{2MNIST} & \multicolumn{2}{c|}{USPS} & \multicolumn{2}{c|}{CIFAR-10} & \multicolumn{2}{c|}{STL-10} & \multicolumn{2}{c|}{CIFAR-100} \\\hline\hline
      
      & ACC & NMI & ACC & NMI & ACC & NMI & ACC & NMI & ACC & NMI & ACC & NMI & ACC & NMI \\\hline
      
      
      DCSS$_{agg}$ & 87.51 & 81.35 & 59.61 & 58.47 & 45.21 & 62.90 & 76.38 & 75.17 & 81.39 & 70.77 & 85.60 & 76.10 & 43.71 & 42.90 \\\hline
      
      DCSS$_u$ & 95.99 & 89.95 & 62.90 & 63.58 & 45.31 & 63.00 & 82.93 & 81.84 & 82.43 & 71.32 & 86.47 & 76.52 & 44.08 & 43.70 \\\hline
    \end{tabular}
  }
\end{table*}
To demonstrate the efficacy of introducing cluster-specific losses and the adopted approach in iteratively updating the network's parameters over K successive runs, as is discussed in Section \ref{step1}, we present a comparative analysis against an alternative approach where all loss terms are aggregated in the initial phase, followed by a single backward pass to update the network parameters at once. The results are presented in Table \ref{newtable} as DCSS$_{agg}$, highlighting a consistent trend. Across all experiments and datasets, the performance of DCSS$_u$ surpasses that of DCSS$_{agg}$  showing the advantages of employing cluster-specific loss functions and iteratively updating the network parameters in K successive runs. \b{To ensure a fair comparison, in all experiments, the number of iterations of DCSS$_{agg}$ is K times greater than the number of iterations in DCSS$_u$.}

\begin{table*}[h]
\caption{ACC and NMI with different extracted features using SimCLR. The second best result is shown by *.}
\label{table2}
\centering 
\begin{tabular}{l c | c | c| c | c } 
\hline
Dataset & Method & \multicolumn{2}{c|}{SimCLR (Z space)} & \multicolumn{2}{c|}{SimCLR (H space)} \\
\hline
& & ACC & NMI & ACC & NMI \\
\hline
\multirow{7}{*}{STL-10} & DEC & 84.26 & 75.14 & 85.66 & 76.13 \\
& IDEC & 84.55 & 75.64 & 85.90 & 76.44 \\
& DCN & 81.26 & 71.23 & 82.84 & 73.35 \\
& DKM & 83.14 & 74.93 & 83.64 & 74.85 \\
& AE$+$k-means & 84.59 & 75.39 & 85.29 & 76.00 \\
& DCSS$_u$ & 85.11* & 76.31* & 86.47* & 76.52* \\
& DCSS & \textbf{86.66} & \textbf{76.66} & \textbf{87.91} & \textbf{77.39} \\
\hline
\multirow{7}{*}{CIFAR-10} & DEC & 76.60 & 66.91 & 81.59 & 70.00 \\
& IDEC & 76.78 & 67.02 & 81.89 & 70.49 \\
& DCN & 73.00 & 62.35 & 78.05 & 67.82 \\
& DKM & 76.99 & 66.52 & 81.19 & 69.47 \\
& AE$+$k-means & 75.13 & 65.98 & 80.23 & 68.65 \\
& DCSS$_u$ & 77.21* & 68.00* & 82.43* & 71.32* \\
& DCSS & \textbf{77.55} & \textbf{68.91} & \textbf{84.16} & \textbf{75.09} \\
\hline
\multirow{7}{*}{CIFAR-100} & DEC & 42.42 & 42.22 & 43.75 & 42.55 \\
& IDEC & 41.83 & 41.20 & 43.87 & 42.73 \\
& DCN & 40.88 & 40.13 & 42.11 & 41.59 \\
& DKM & 42.36 & 42.16 & 43.40 & 42.90 \\
& AE$+$k-means & 42.56 & 42.26 & 43.81 & 42.86 \\
& DCSS$_u$ & 43.12* & 42.88* & 44.08* & 43.70* \\
& DCSS & \textbf{43.43} & \textbf{43.18} & \textbf{45.10} & \textbf{44.51} \\
\hline
\end{tabular}
\end{table*}
\subsubsection{Effect of input features}

To make sure our improved performance isn't solely due to using features from SimCLR as input for our algorithm, we conducted an additional experiment which presents the performance of the other deep learning-based methods that can accept the extracted features, here SimCLR, as the input of their algorithm. To this end, we evaluated the clustering performance using two types of features from SimCLR. In the first experiment, named SimCLR (Z space), we used the output of the projection head as input for DCSS and the other algorithms. In the second experiment, SimCLR (H space), we used the output of ResNet-34 (i.e. the SimCLR's backbone), after the average pooling layer, as input. The results are shown in Table \ref{table2}. The best results for each dataset are in bold, and the second-best results are marked with an asterisk (*). Notably, our DCSS method consistently outperforms all other algorithms across all datasets. Additionally, DCSS$_u$ consistently ranks as the second-best method in all experiments.

\subsubsection{DCSS as a General Framework}\label{general_framework}
\begin{table*}[h]
  \caption{ACC and NMI on the benchmark datasets when employing DCSS as a general framework to improve state-of-the-art AE-based clustering methods.}
  \label{table3}
  \centering
  \scalebox{0.6}{
    \begin{tabular}{|c||*{16}{c|}}
      \hline
      \backslashbox[4em]{Method}{Datasets} & \multicolumn{2}{c|}{MNIST} & \multicolumn{2}{c|}{Fashion MNIST} & \multicolumn{2}{c|}{2MNIST} & \multicolumn{2}{c|}{USPS} & \multicolumn{2}{c|}{CIFAR-10} & \multicolumn{2}{c|}{STL-10} & \multicolumn{2}{c|}{CIFAR-100} \\
      \hline\hline
      & ACC & NMI & ACC & NMI & ACC & NMI & ACC & NMI & ACC & NMI & ACC & NMI & ACC & NMI \\
      \hline
      DEC+MNet & 89.13 & 86.97 & 61.25 & 56.30 & 44.25 & 57.35 & 77.58 & 78.15 & 81.88 & 70.79 & 85.97 & 76.33 & 44.10 & 42.81 \\
      IDEC+MNet & 90.51 & 85.42 & 60.12 & 57.16 & 44.83 & 58.00 & 76.58 & 78.14 & 82.26 & 70.86 & 86.11 & 76.35 & 44.21 & 43.01 \\
      DCN+MNet & 87.49 & 83.25 & 54.23 & 58.69 & 45.62 & 48.24 & 76.90 & 77.59 & 79.13 & 68.51 & 83.11 & 74.00 & 43.21 & 41.86 \\
      DKM+MNet & 88.31 & 84.52 & 57.23 & 56.26 & 44.34 & 49.50 & 77.13 & 78.02 & 81.30 & 70.63 & 83.75 & 75.01 & 43.78 & 43.19 \\
      \hline
    \end{tabular}
  }
\end{table*}






In this section, we demonstrate the effectiveness of the DCSS method as a general framework where the $\mathbf{u}$ space is trained with other AE-based clustering techniques. To this end, we substitute the first phase, presented in Section \ref{step1}, with other deep learning-based techniques that train an effective subspace using an AE for the purpose of data clustering. Among our comparison methods, DEC, IDEC, DCN, and DKM algorithms are AE-based. For each dataset, we train AEs using these algorithms, then take their encoder part and append our proposed MNet to the latent space. Then, we run the second phase of DCSS. Results of such implementation are reported in Table \ref{table3} where X+MNet indicates the performance of DCSS employing the X method's latent space as the DCSS's $\mathbf{u}$ space. Note that, for the RGB datasets, features are constructed using the pre-trained SimCLR (H space).

Comparing the clustering results reported in Table \ref{table1}, Table \ref{table2}, and Table \ref{table3} confirms the effectiveness of DCSS as a general framework to improve the existing state-of-the-art AE-based clustering methods. On average, MNet improves the clustering performance of DEC, IDEC, DCN, and DKM respectively by 2.68\% (1.66\%), 2.23\% (1.21\%), 2.58\% (1.35\%), and 2.12\% (1.23\%) in terms of ACC (NMI). 

\subsubsection{\b{Effectiveness of MNet in Improving Self-supervised Learners}}

\b{In this section, we evaluate the efficacy of MNet in enhancing MoCo's performance on CIFAR-10 and CIFAR-100 datasets for clustering tasks. Table \ref{rr} presents the outcomes of applying MNet to features extracted by a self-supervised learners, with updates to the encoder network and MNet parameters during training using the discussed loss in Section \ref{2&3}. Since MoCo is primarily designed for representation learning rather than clustering, we initially apply the k-means algorithm to assess ACC and NMI, denoted as MoCo + Kmeans in Table \ref{rr}. Subsequently, we compute ACC and NMI in the latent space of MNet, represented as MoCo + MNet in the same table.}

\b{The comparison between MoCo + MNet and MoCo + Kmeans clearly demonstrates the effectiveness of MNet in enhancing ACC and NMI for CIFAR-10 and CIFAR-100. This experiment highlights that instead of relying solely on the initial phase, a self-supervised learner like MoCo could be substituted, with MNet having the potential to improve their performance. This is due to MNet's ability to explore relationships among samples within the batch, a feature not present in MoCo's design. }

\begin{table*}[h]
  \caption{\b{ACC and NMI on CIFAR-10 and CIFAR-100 when employing MNet to improve MoCo.}}
  \label{rr}
  \centering
  \scalebox{0.8}{
    \begin{tabular}{|c||*{16}{c|}}
      \hline
      \backslashbox[4em]{Method}{Datasets} &  \multicolumn{2}{c|}{CIFAR-10} & \multicolumn{2}{c|}{CIFAR-100} \\
      \hline\hline
      & ACC & NMI & ACC & NMI  \\
      \hline
      MoCo + Kmeans & 74.11 &  66.22 & 43.18 & 42.75\\ \hline
      MoCo + MNet & 76.24 & 67.00 & 44.01 & 43.22\\
      \hline
    \end{tabular}
  }
\end{table*}

\subsubsection{Hyperparameters Sensitivity}
\begin{figure*}[h]
  \centering
  \includegraphics[width=1\linewidth]{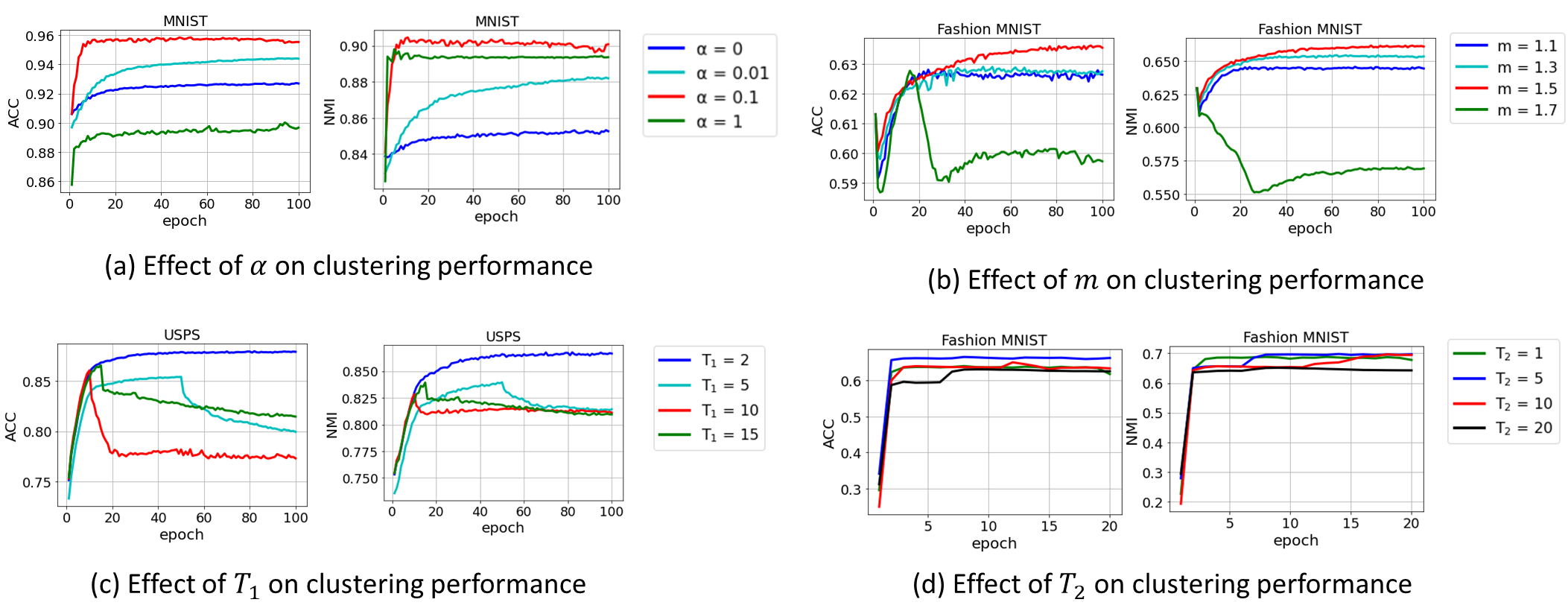}
  \caption{Sensitivity of DCSS to different hyperparameters.}
\label{hyperparameters}
\end{figure*}

\begin{figure*}[h]
  \centering
  \includegraphics[width=1\linewidth]{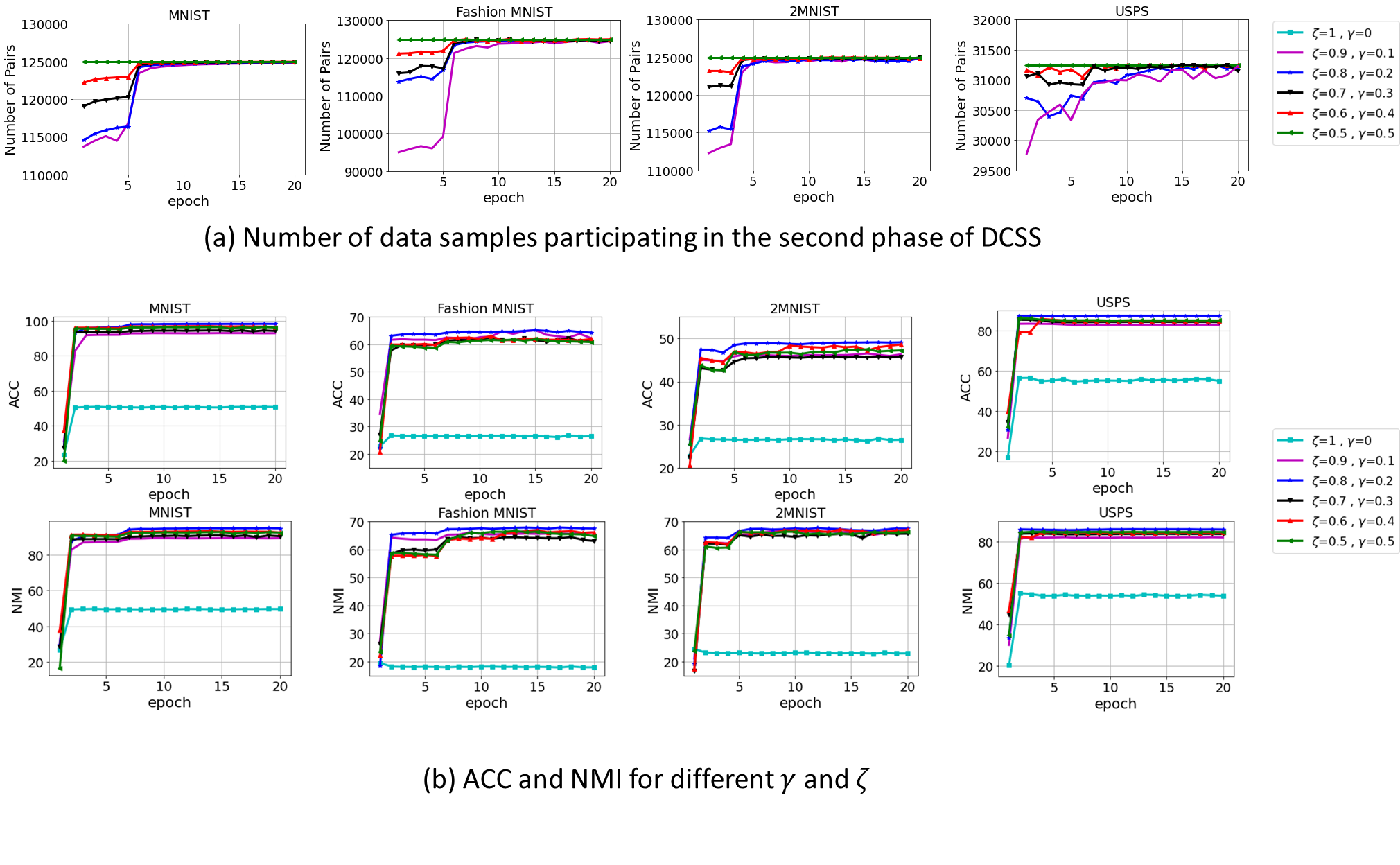}
  \caption{Changing hyperparameters $\zeta$ and $\gamma$ for different datasets. (a) Number of data pairs participating in the second phase of DCSS. \s{In the first epochs, MNet can barely recognize the relation between data points using soft group assignments defined in the $\mathbf{u}$ space. After completing the first $T_2=5$ epochs, we switch to measure similarities in the more reliable space $\mathbf{q}$; hence the number of participating pairs increases dramatically.  In the final epochs for different values for $\zeta$ and $\gamma$, almost all pairs contribute to the loss function defined in \eqref{eq5}. For the maximum ambiguity region 1, i.e. $\zeta = 1$ and $\gamma = 0$, all pairs do not participate in the second phase of DCSS and stay on the ambiguity region. Hence, 0 pairs contributes to this phase.} (b) Clustering performance in terms of ACC and NMI for different datasets for different values of $\zeta$ and $\gamma$; the clustering performance of DCSS is less sensitive to the choice of $\zeta$ and $\gamma$ in the range of [0.5,0.9] and [0.1,0.5], respectively.}
\label{changing_u_l}
\end{figure*}

In Fig. \ref{hyperparameters}, we investigate the effect of different hyperparameters on DCSS clustering performance. For hyperparameters of the first phase (i.e. $\alpha$, $m$, and $T_1$), we report the performance of clustering using DCSS$_u$ (as is presented in Section \ref{ClusteringPerformance}).

In our proposed method, hyperparameters are fixed across all datasets, i.e. no fine-tuning is performed per dataset. Hence, one may obtain more accurate results by tuning the hyperparameters per dataset.

In Fig. \ref{hyperparameters}(a), we explore the importance of the centering loss in the first phase's loss function, shown in \eqref{new_eq1}, by changing $\alpha \in \{0,0.01,0.1,1\}$ for MNIST dataset. As is shown in this figure, by increasing the value of $\alpha$ from 0 to 0.1, our DCSS performance significantly enhanced in terms of ACC and NMI, which demonstrates the effectiveness of incorporating the centering loss beside the reconstruction loss in the first phase's loss function. We observed a similar trend across all the other datasets. In all our experiments, for all datasets, $\alpha$ is set to $0.1$. 

Fig. \ref{hyperparameters}(b) shows the impact of the level of fuzziness $m$ on the clustering performance of DCSS$_u$ for the Fashion MNIST dataset, where $m\in \{1.1,1.3,1.5,1.7\}$. In the case of m$\rightarrow$ 1 (m$\rightarrow \infty$), group membership vectors converge to one-hot (equal probability) vectors. As shown in this figure, as desired, the DCSS method is not too sensitive to m when it is set to a reasonable value. We observed a consistent pattern across all the remaining datasets. In all our experiments, for all datasets, m is set to $1.5$.

In Fig. \ref{hyperparameters}(c), we scrutinize the effect of update interval $T_1$ in the clustering performance of the first phase for $T_1\in \{2,5,10,15\}$. As is expected, better clustering performance in terms of ACC and NMI is acquired for a smaller value of $T_1$ for the USPS dataset. Consistently, the same behavior was observed across all remaining datasets. In our experiments, for all datasets, $T_1$ is set to 2. 

In Fig. \ref{hyperparameters}(d), we change the number of training epochs $T_2$, defined in Section \ref{2&3}, for Fashion MNIST dataset, where $T_2 \in \{1,5,10,20\}$. As is expected, for a very small $T_2$ value, e.g. $T_2 = 1$, where training the $\mathbf{q}$ space is mainly supervised by the $\mathbf{q}$ space itself even at the MNet training outset, DCSS cannot provide a proper $\mathbf{q}$ space, since $\mathbf{q}$ is not a sufficiently reliable space to be used for self-supervision. The figure also shows that for a very large $T_2$ value, e.g. $T_2=20$, when we only trust the $\mathbf{u}$ space for supervising the $\mathbf{q}$ space, we cannot train an effective $\mathbf{q}$ space. As shown, a good clustering performance can be obtained when $T_2$ is set to a moderate value. In our experiments, for all datasets, $T_2$ is set to $5$. This demonstrates the effectiveness of the proposed strategy in supervising the MNet training using both the $\mathbf{u}$ and $\mathbf{q}$ spaces.

In Fig. \ref{changing_u_l}, we change $\zeta$ and $\gamma$ in range [0,1], where $\zeta+\gamma=1$, to observe model convergence and accuracy for different lengths of the ambiguity interval, defined as $\zeta-\gamma$, ranging from 1 (when $\zeta=1$) to 0 (when $\zeta=0.5$). Fig. \ref{changing_u_l}(a) shows the number of pairs participating in minimizing the loss functions defined in \eqref{eq4} and \eqref{eq5}. As can be seen, at the beginning of the second phase, our model can make a decisive decision only about a few pairs, and the remaining pairs are in the ambiguous region. As the second phase of the training process progresses, more and more pairs are included in the loss functions optimization process. Finally, at the end of the second phase, almost all pairs contribute to the training.

Furthermore, in Fig. \ref{changing_u_l}(b), we investigate the influence of $\zeta$ and $\gamma$ in clustering performance. As it can be seen, as is desired, the final clustering performance of our DCSS framework is not highly sensitive to the choice of $\zeta$ and $\gamma$ when these are set to reasonable values. In all our experiments, $\zeta=0.8$ and $\gamma = 0.2$ for all experiments and datasets.

\begin{figure*}[h]
  \centering
  \includegraphics[width=0.45\linewidth]{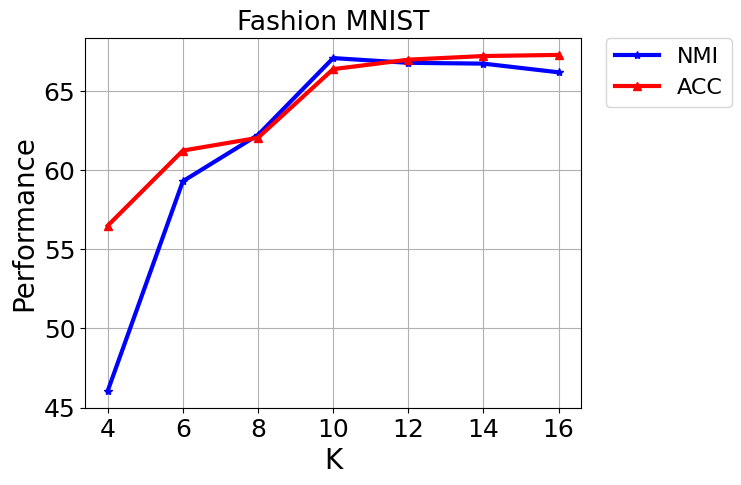}
  \caption{Sensitivity of DCSS model to the initial number of clusters (K).}
\label{k_exp}
\end{figure*}

In Figure \ref{k_exp}, we investigate the impact of initializing the number of clusters to values smaller or larger than the actual number of clusters on the performance of our DCSS model using the Fashion MNIST dataset, where the true number of clusters is 10. When the initial number of clusters (e.g., K = 4) is smaller than the actual number, the model is forced to group data samples from multiple clusters into a single cluster, resulting in a decrease in the DCSS model's performance in terms of ACC and NMI. Conversely, with higher values of K (e.g., K = 16), the algorithm may tend to over-segment the data, creating smaller clusters that may not well align with the true structure of the data. In an extreme scenario where K is set to a very large value, the over-segmenting of the data may cause overfitting to the training data. Employing a validation set becomes valuable when conducting a parameter sweep to identify an appropriate value for K. See \cite{kodinariya2013review, milligan1985examination} for more discussion on how to have a prior estimation for the number of clusters.    

\subsection{Features visualization}
In order to investigate the effectiveness of our model in extracting useful features for different datasets, we train a deep neural network with the same structure as is presented in Section 1 of the supplementary material file in a \emph{supervised} manner, and then we compare the output of the first convolutional layers for the trained model and our proposed DCSS model. As it can be seen in Fig. \ref{Filter}, our DCSS learns a variety of low- and high-frequency features, which are similar to features learned in a supervised manner. This demonstrates the effectiveness of our framework in finding informative features in an unsupervised manner.

\begin{figure}[H]
  \centering
  \includegraphics[width=0.8\linewidth]{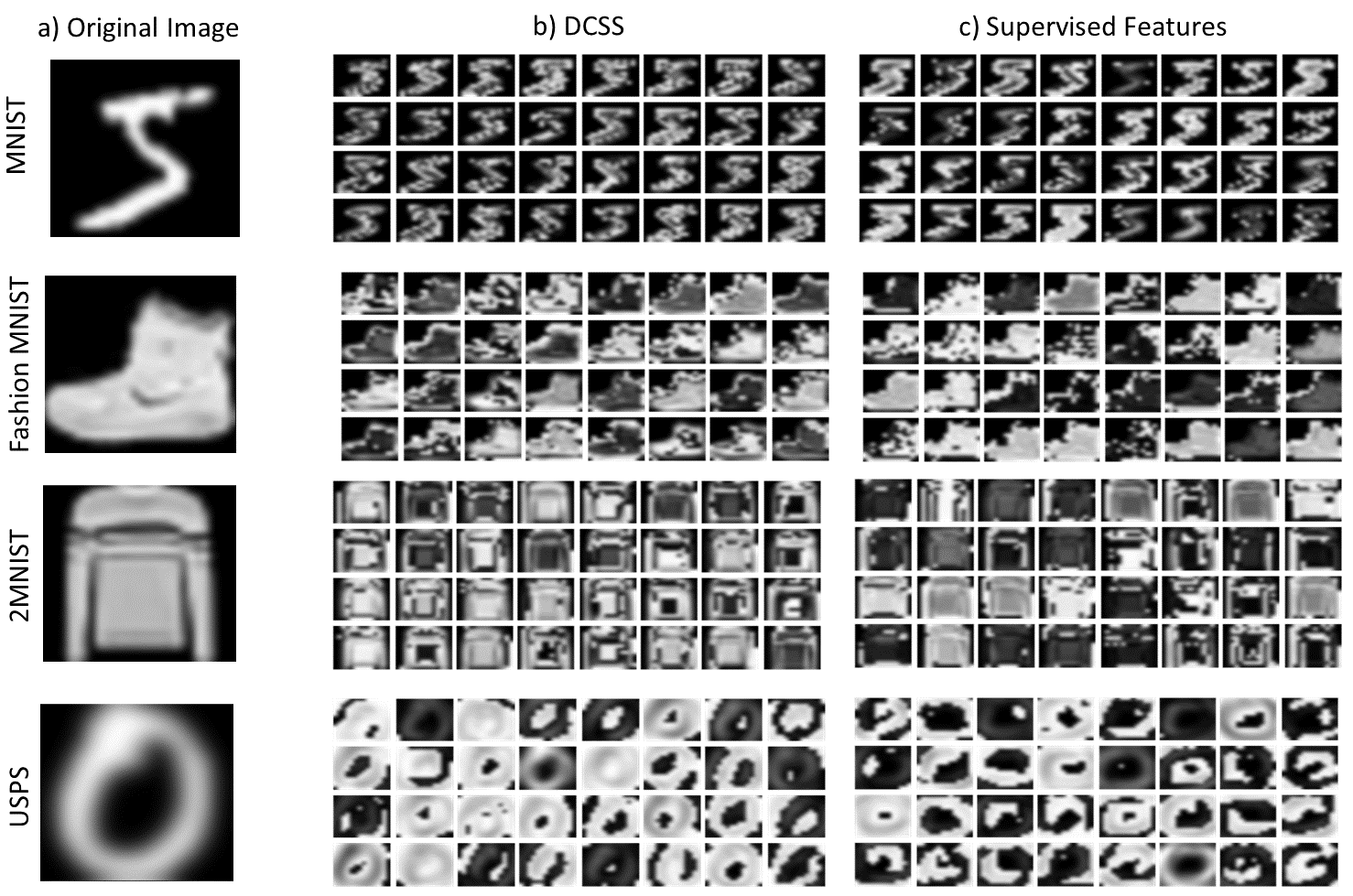}
  \caption{(a) samples of MNIST, Fashion MNIST, 2MNIST, and USPS. The output of the first convolutional layer using (b) the unsupervised DCSS method and (c) a supervised manner employing the same network structure as of DCSS is shown in Section 1 of the supplementary material file. }
\label{Filter}
\end{figure}

\section{Conclusion}\label{conclusion}

 In this paper, we present a novel, effective, and practical method for data clustering employing the novel concept of self-supervision with pairwise similarities. Despite most clustering methods, the proposed DCSS algorithm employs soft cluster assignments in its loss function, optimizes cluster-specific losses, and takes advantage of the relevant information available in the sample pairs. The proposed algorithm is shown to perform well in practice compared to previous state-of-the-art clustering algorithms. We also show that the DCSS's self-supervision approach can be employed as a general approach to improve the performance of state-of-the-art AE-based clustering methods. 

\bibliography{sn-bibliography}

\end{document}